\documentclass{article}

\PassOptionsToPackage{numbers,sort&compress}{natbib}

\usepackage[main,preprint]{neurips_2026}
\usepackage[utf8]{inputenc}
\usepackage[T1]{fontenc}
\usepackage{hyperref}
\usepackage{url}
\usepackage{amsmath,amssymb,amsfonts}
\usepackage{amsthm}
\usepackage{graphicx}
\usepackage{subcaption}
\usepackage{mathtools}
\usepackage{array}
\usepackage{multirow}
\usepackage{booktabs}
\usepackage{nicefrac}
\usepackage{microtype}
\usepackage{xcolor}
\usepackage{algorithm}
\usepackage{algpseudocode}

\newtheorem{theorem}{Theorem}

\newtheorem{definition}{Definition}
\newtheorem{assumption}{Assumption}
\newtheorem{proposition}{Proposition}

\newtheorem{lemma}{Lemma}

\usepackage{tabularx}
\usepackage{makecell}
\usepackage{threeparttable}

\newcommand{\Sset}{\mathcal{S}}
\newcommand{\Aset}{\mathcal{A}}
\newcommand{\Cset}{\mathcal{C}}
\newcommand{\PiD}{\Pi_{\mathrm{det}}}

\newcommand{\ind}{\mathbb{I}}
\newcommand{\eps}{\varepsilon}

\newcommand{\ThetaSet}{\Theta}

\begin{document}
\title{Learning Chance-Constrained MDPs with Bellman Distributional Certificates}

\author{Chenbei Lu\\
  Cornell University AI for Science Institute\\
  Cornell University\\
  Ithaca, NY, USA\\
  \texttt{chenbei.lu@cornell.edu}
  \And
  Hongyu Yi\thanks{Corresponding author.}\\
  Department of Electrical \& Computer Engineering\\
  University of Washington\\
  Seattle, WA, USA\\
  \texttt{hyyi@uw.edu}
}

\maketitle

\begin{abstract}
Safe reinforcement learning (RL) commonly enforces expected-cost constraints, but such expectation-form safety may fail to control the probability of rare high-cost trajectories. Chance-constrained MDPs (CCMDPs) impose a stronger probability-level requirement, but are widely viewed as harder because the chance constraint itself is nonconvex and depends on the full trajectory rather than a Bellman-linear expectation. In this paper, we reveal that this computational difficulty does not necessarily imply a higher statistical price. For tabular discounted CCMDPs with bounded successor support and access to a certified planning oracle, we establish a model-based upper bound, with a matching lower bound up to logarithmic terms. Technically, our key idea is the \emph{Bellman distributional certificate}, which constructs a Bellman recursion for constraint violation probabilities before policy selection. The certificate reduces trajectory-level safety analysis to local transition uncertainty, avoiding a union bound over the potentially exponential number of candidate policies. For stochastic policies, we give a model-free variance-reduced policy-gradient algorithm with a finite-sample expected KKT-residual guarantee and independent validation of every accepted policy. Numerical experiments on synthetic CCMDPs and an IEEE 14-bus energy storage control benchmark illustrate the safety and mechanism behavior of the proposed algorithms.
\end{abstract}

\section{Introduction}
\label{sec:intro}
Reinforcement learning (RL) is a central framework for online decision-making in uncertain environments \citep{Sutton2018RLBook}. In safety-critical applications, however, maximizing expected reward is not sufficient: the agent must also satisfy certain safety or reliability requirements, the central concern in safe reinforcement learning \citep{GarciaFernandez2015SafeRLSurvey,Wachi2024survey,Achiam2017CMDP_policy_opt}. A common formulation for safe reinforcement learning is constrained Markov decision processes (CMDPs), which maximize the cumulative reward subject to constraints on expected cumulative costs \citep{Derman1972ConstrainedChains,BeutlerRoss1985ConstrainedChains,Altman1999ConstrainedMDPs}. This expectation-form constraint is mathematically convenient: expected-cost constraints are linear in the occupancy measure, so a CMDP with known transitions can be formulated as a linear program~\citep{Altman1999ConstrainedMDPs,FeinbergShwartz1996ConstrainedDiscounted,Vaswani2022CMDPSampleComplexity}. However, CMDPs control safety only through expectations. They do not directly control the probability distribution of rare high-cost events such as overloads or failures \citep{Chow2018Percentile,borkar2014risk}.

Chance-constrained Markov decision processes (CCMDPs) provide a stronger probability-level formulation of safety \citep{Ono2015CCDP,Pfrommer2022safe}. Instead of constraining only the average cumulative cost, a CCMDP requires the cumulative trajectory cost to remain below a prescribed threshold with high probability \citep{Chen2024ProbabilisticConstraint}. In the discounted setting, this takes the form
\begin{equation*}
    \mathbb P_{\pi}\left(\sum\nolimits_{t=0}^{\infty}\gamma^t c(s_t,a_t)\le d\right)\ge 1-\delta.
\end{equation*}
Unlike an expected-cost constraint, this requirement directly limits the probability of high-cost trajectories and therefore better captures reliability requirements in applications such as energy systems, transportation, and financial markets \citep{DallAnese2017ChanceOPF,Zhao2018IntermodalChance,Beraldi2022EnhancedIndexation}. However, this stronger safety object appears much harder to learn. A chance constraint is generally nonconvex and depends on the full trajectory distribution rather than an expectation of additive costs, so it breaks the Bellman-linear and occupancy-measure structure used in CMDP analyses, which leads to a generally nonconvex feasible set \citep{Altman1999ConstrainedMDPs,Vaswani2022CMDPSampleComplexity,Shen2024FlippingCCMDP}.

Learning CCMDPs is challenging for two closely related reasons. First, chance constraints restrict the distribution of cumulative trajectory costs rather than their expectations. This breaks the Bellman-linear and occupancy-measure structure that makes CMDPs tractable, and the resulting feasible set is generally nonconvex \citep{Altman1999ConstrainedMDPs,Vaswani2022CMDPSampleComplexity,Shen2024FlippingCCMDP}. Second, without a reusable certificate, one must estimate violation probabilities for each policy separately, leading to inefficient sample reuse \citep{Yi2025LCSS}. Addressing these challenges is essential for developing sample-efficient RL algorithms with reliable safety guarantees. See Appendix \ref{app:related} for a detailed literature review.

To overcome these challenges, our contributions can be summarized as follows.

\paragraph{Bellman distributional certificates for chance constraints.}
We introduce Bellman distributional certificates, which track a remaining safety budget to express the chance-constraint violation probability through a Bellman recursion. This replaces separate policy-level safety checks with reusable one-step Bellman updates for violation probabilities, which is the key to enable sample-efficient constraint feasibility certification across candidate policies. 

\paragraph{Model-based learning with Bellman certificates.}
We develop a model-based algorithm for learning a near-optimal deterministic policy with empirical Bellman distributional certificates in CCMDPs. Under Assumption~\ref{ass:sharp-certified-planning}, which requires bounded successor support and a certified planning oracle, we prove that the returned policy is chance constraint-feasible and near-optimal with sample complexity \(\widetilde O(|\Sset||\Aset|[(1-\gamma)^{-3}\eps^{-2}+(1-\gamma)^{-1}\rho^{-2}])\), which significantly improves the best-known deterministic CCMDP learning result \citep{Yi2025LCSS}. We also prove a matching lower bound up to logarithmic terms, showing that the sample requirements are unavoidable. The proof uses shared row-wise reverse-KL confidence sets and a trajectory-level KL transfer argument, which control the safety of data-dependent policies without a union bound over exponentially many policies.

\paragraph{Model-free policy-gradient algorithm with trajectory certificates.}
We further develop a model-free algorithm for stochastic policy learning using variance-reduced policy gradients and trajectory certificates. Under standard regularity assumptions, the optimization phase attains an expected KKT-residual guarantee with \(\widetilde O(\eps_{\rm kkt}^{-3})\) samples. To our knowledge, this is the first result for finite sample stochastic policy learning with original chance constraints \citep{Chow2018Percentile,Chen2024ProbabilisticConstraint,Yuan2020STORMPG,Zhang2021TSIVRPG}.

\section{Chance-Constrained Markov Decision Processes}
\label{sec:problem}
Consider a standard infinite-horizon discounted chance-constrained MDP $\mathcal M=\langle \Sset,\Aset,P,r,\gamma,\Cset,\mu\rangle$.
Here $\Sset$ is a finite state space, $\Aset$ is a finite action space, $P:\Sset\times\Aset\to\Delta(\Sset)$ is the unknown transition kernel, $r:\Sset\times\Aset\to[0,1]$ is the reward function, $\gamma\in(0,1)$ is the discount factor, and $\mu$ is the initial distribution of the state.
The set $\Cset=\{(c_i,d_i,\delta_i)\}_{i=1}^m$ specifies the following chance constraints:
\begin{align}
\mathbb P_{P,\pi,\mu}\!\left(\sum\nolimits_{t=0}^{\infty}\gamma^tc_i(s_t,a_t)\le d_i\right)
\ge 1-\delta_i,
\qquad\forall i\in[m].
\end{align}
where $c_i:\Sset\times\Aset\to[0,1]$ is the instantaneous safety cost, $d_i$ is the safety budget, and $\delta_i$ is the allowed violation probability.

Since the chance constraints depend on the full trajectory, we allow the policy $\pi$ to also depend on trajectory parameters in addition to the current state. Specifically, a policy $\pi:\Sset\times\Gamma\to\Delta(\Aset)$ maps current state $s$ and trajectory parameters $g\in\Gamma$ to the distribution simplex over actions. Ordinary Markov policies correspond to the special case where \(\Gamma\) is a singleton.

For any policy $\pi$, let $\mathbb P_{P,\pi,\mu}$ and $\mathbb E_{P,\pi,\mu}$ denote probability and expectation for trajectories generated by $P$, policy $\pi$, and initial distribution $\mu$, including updates of trajectory parameters when present. We define
\begin{equation}
q_{P,i}(\pi)=
\mathbb P_{P,\pi,\mu}\!\left(\sum\nolimits_{t=0}^{\infty}\gamma^tc_i(s_t,a_t)>d_i\right),
\qquad
J_P(\pi)=
\mathbb E_{P,\pi,\mu}\!\left[\sum\nolimits_{t=0}^{\infty}\gamma^t r(s_t,a_t)\right].
\end{equation}
For a policy class $\Pi$, the feasible set and optimal value are
\begin{equation}
\Pi^\star=\{\pi\in\Pi:q_{P,i}(\pi)\le\delta_i,\ i\in[m]\},
\qquad
V^\star_{\rm cc}=\sup\nolimits_{\pi\in\Pi^\star}J_P(\pi).
\end{equation}

\section{Bellman Distributional Certificates}
\label{sec:certificates}
\textcolor{black}{We construct upper confidence bounds on chance-constraint violation probabilities. We first truncate the discounted tail and round the remaining budget to obtain a finite violation event. We then bound its probability using a robust Bellman recursion.}

\subsection{Violation-Probability Bellman Table}
\textcolor{black}{Estimating violation probabilities policy by policy is statistically inefficient: samples used to validate one completed policy do not directly certify another, while the deterministic policy class can be exponentially large. We instead evaluate each policy with a Bellman table whose entries are violation probabilities from a state, a rounded remaining budget, and a time index.}

\textcolor{black}{For a tail allowance $\alpha_{\rm tail}$ and grid widths $\eta_i>0$, define}
\begin{equation}
H_\alpha=
\left\lceil
\frac{\log(1/((1-\gamma)\alpha_{\rm tail}))}{1-\gamma}
\right\rceil,
\qquad
b_i^0=\left\lfloor\frac{d_i-\alpha_{\rm tail}}{\eta_i}\right\rfloor,
\qquad
w_{i,h}(s,a)=\left\lceil\frac{\gamma^h c_i(s,a)}{\eta_i}\right\rceil .
\label{eq:certificate-discretization}
\end{equation}
Here $H_\alpha$ makes the discounted tail after the certificate horizon at most $\alpha_{\rm tail}$, $b_i^0$ is the initial integer budget, and $w_{i,h}(s,a)$ is the integer charge of taking action $a$ in state $s$ at time index $h$. The initial budget is rounded down and each one-step charge is rounded up, so the finite violation event is conservative for the original chance constraint. The finite budget state is $b=(b_1,\ldots,b_m)$, where each coordinate takes integer values $0,\ldots,b_i^0$ together with the failure value $-1$; if $b_i^0<0$, the coordinate starts at $-1$. We write $\mathcal B$ for the product of these coordinate sets.
The deterministic update used below, denoted $B_h(s,a,b)$, subtracts the charges $w_{i,h}(s,a)$ coordinatewise and sets any exhausted coordinate to $-1$. Below we write $H=H_\alpha$ when the tail allowance is fixed.

\begin{definition}[Bellman violation-probability table]
For a fixed certificate horizon \(H\), policy \(\pi\), constraint \(i\), and
transition kernel \(P\), define \(q^\pi_{P,i}(s,b,h)\) as the conditional
probability that the \(i\)-th rounded budget coordinate becomes negative by the
certificate horizon, when the rounded process starts from original state \(s\),
rounded remaining-budget vector \(b\), and time index \(h\).
\end{definition}

\begin{lemma}[Markov sufficiency of rounded budgets]
\label{lem:certificate-parameter-sufficiency}
Fix a certificate horizon \(H\), budget grid \(\eta\), and a deterministic policy
\(\pi\) whose actions are selected as \(a_h=\pi_h(s_h,b_h)\). Then, for each
constraint \(i\), the conditional probability of rounded finite-prefix violation
after time \(h\), given the trajectory history \(\mathcal F_h\), depends on
\(\mathcal F_h\) only through \((s_h,b_h,h)\). In particular, there exists a
function \(q^\pi_{P,i}\) such that
\begin{equation}
\mathbb P_{P,\pi}\!\left(\exists t\in\{h,\ldots,H\}: b_{i,t}<0
\mid \mathcal F_h\right)
=
q^\pi_{P,i}(s_h,b_h,h).
\end{equation}
Moreover, \(q^\pi_{P,i}\) satisfies the Bellman recursion
\begin{equation}
q^\pi_{P,i}(s,b,h)
=
\sum\nolimits_{s'}P(s'\mid s,\pi_h(s,b))
q^\pi_{P,i}(s',B_h(s,\pi_h(s,b),b),h+1),
\end{equation}
with boundary condition $q^\pi_{P,i}(s,b,H)=\ind\{b_i<0\}$.
\end{lemma}

The proof is given in Appendix~\ref{app:certificate-parameter-sufficiency-proof}.
For a fixed policy $\pi_h(s,b)$, the lemma shows that $(s,b,h)$ suffices to evaluate the finite violation probability. For fixed $H$ and grid eta, let $\Pi_{H,\eta}$ denote the deterministic policies that use $\pi_h(s,b)$ for $h = 0,...,H-1$. All policies in this class follow the same deterministic continuation, fixed in advance, from time H onward.

\subsection{Pessimistic Bellman Certificates}
Lemma~\ref{lem:certificate-parameter-sufficiency} gives the exact Bellman recursion for the finite violation probability when the transition kernel is known. In learning, \(P\) is unknown, so the certificate maximizes each Bellman backup over a row-wise confidence set. Since the table represents violation probabilities, this robust backup overestimates the true violation probability whenever the true transition row belongs to the confidence set.

\begin{definition}[Pessimistic Bellman distributional certificate]
Fix a policy $\pi\in\Pi_{H,\eta}$, an empirical transition kernel $\widehat P$, and row-wise confidence sets $\{\mathcal C_{s,a}\}_{s,a}$.
For each constraint $i$, define the empirical violation table
$\overline q_i^\pi(s,b,h)$ backward by $\overline q_i^\pi(s,b,H)=\ind\{b_i<0\},$
and for $h=H-1,\ldots,0$ by
\begin{equation}
\overline q_i^\pi(s,b,h)
=
\sup_{p\in\mathcal C_{s,a}}
p^\top\overline q_i^\pi(\cdot,B_h(s,a,b),h+1),
\qquad a=\pi_h(s,b),
\label{eq:pessimistic-backup}
\end{equation}
The pessimistic Bellman distributional certificate of $\pi$ is $\overline q_{i,H,\eta}(\pi)
=
\sum\nolimits_s \mu(s)\overline q_i^\pi(s,b^0,0)$ for all $i\in[m].$
\end{definition}

\textcolor{black}{When the confidence sets contain all true transition rows, conservative rounding and the robust recursion give $q_{P,i}(\pi)\le\mathbb P_{P,\pi,\mu}(b_{i,H}<0)\le\overline q_{i,H,\eta}(\pi)$.}

\begin{proposition}[Certified safety, Proof in Appendix~\ref{app:certified-safety-proof}]
\label{prop:certified-safety}
With high probability over the shared row samples, simultaneously for every $\pi\in\Pi_{H,\eta}$, $\overline q_{i,H,\eta}(\pi)\le \delta_i$ for every $i\in[m]$ implies that $\pi$ is feasible for the original infinite-horizon chance-constrained MDP.
\end{proposition}
\paragraph{Why the certificate is reusable.}
The key statistical point is that one confidence set is constructed for each original transition row and reused at every time and budget state. On the simultaneous row-coverage event, the robust recursion is valid for every policy, including a policy selected after observing the empirical model. Thus the statistical analysis needs neither fresh samples for different Bellman continuations nor a union bound over complete policies.

\section{Model-Based Learning with Bellman Certificates}
\label{sec:deterministic}
This section develops a model-based learning algorithm built on Bellman distributional certificates. We first describe the learning procedure and then present its finite-sample guarantee together with a separate statistical lower bound.

\subsection{Empirical Bellman-Certified Learning}
\label{sec:three-stage}
Bellman-certified model-based learning has three steps: environment model estimation, policy-wise Bellman value and certificate computation, and optimization over certified policies.

\paragraph{Model estimation.} We assume access to a standard generative model. At each original state-action pair $(s,a)$, the learner draws $n$ independent next-state samples from $P(\cdot\mid s,a)$ and forms one empirical row $\widehat P_{s,a}$. The same empirical kernel is reused across all time indices, budget states, constraints, and reward comparisons. The row reward $r(s,a)$ and the costs are known bounded functions. With the known support bound $d_0$ from Assumption~\ref{ass:sharp-certified-planning}, define
\begin{equation}
\kappa_n
:=
\frac{(d_0-1)\log(n+1)+\log(|\Sset||\Aset|/\zeta)}{n},
\qquad
\mathcal C_{s,a}
:=
\{p\in\Delta(\Sset):\mathrm{KL}(\widehat P_{s,a}\|p)\le\kappa_n\}.
\label{eq:kl-confidence-set}
\end{equation}
We use KL from the empirical distribution to the candidate distribution so that the confidence set can include models that assign positive probability to unobserved successors.

\paragraph{\textcolor{black}{Policy evaluation.}} \textcolor{black}{For a deterministic policy $\pi$, we compute its safety certificates using the KL sets in \eqref{eq:kl-confidence-set} and the robust recursion in \eqref{eq:pessimistic-backup}.}

\textcolor{black}{We estimate its return under the shared empirical kernel.} Set $\widehat V_H^\pi(s,H,b)=0$ and, for $h=H-1,\ldots,0$, compute
\begin{equation}
\widehat V_H^\pi(s,h,b)
=
r(s,\pi_h(s,b))
+\gamma\,
{\widehat P}(\cdot\mid s,\pi_h(s,b))^\top
\widehat V_H^\pi(\cdot,h+1,B_h(s,\pi_h(s,b),b)),
\label{eq:reward-backup}
\end{equation}
and define $\widehat J(\pi)=\sum\nolimits_s\mu(s)\widehat V_H^\pi(s,0,b^0)$.
\textcolor{black}{Thus the robust Bellman distributional certificate checks feasibility, while ordinary empirical Bellman evaluation ranks the feasible candidates.}

\paragraph{Finite certified problem.}
\textcolor{black}{We maximize estimated return over the solver's finite policy class $\Pi_{\mathcal O}$, subject to the certified chance constraints:}
\begin{equation}
\max\nolimits_{\pi\in\Pi_{\mathcal O}}\widehat J(\pi)
\qquad
\text{s.t.}\quad
\overline q_{i,H_\alpha,\eta}(\pi)\le\delta_i,\quad i\in[m].
\label{eq:finite-certified-problem}
\end{equation}

\begin{definition}[$\xi$-optimal finite certified solver]
\label{def:finite-certified-solver}
For $\xi\ge0$, whenever the certified feasible set is nonempty, a finite certified solver returns a policy $\widehat\pi\in\Pi_{\mathcal O}$ and its Bellman reward and safety tables such that $\overline q_{i,H_\alpha,\eta}(\widehat\pi)\le\delta_i$ for all $i\in[m]$ and
\begin{equation}
\widehat J(\widehat\pi)
\ge
\max\nolimits_{\pi\in\Pi_{\mathcal O}:\ \overline q_{i,H_\alpha,\eta}(\pi)\le\delta_i,\ i\in[m]}
\widehat J(\pi)-\xi.
\end{equation}
When $\xi=0$, the solver is exact.
\end{definition}

{
\begin{assumption}[Bounded successor support and certified planning]
\label{ass:sharp-certified-planning}
There is a known integer $d_0\ge2$, independent of $|\Sset|$, $|\Aset|$, $\gamma$, $\eps$, and $\rho$, such that
\begin{equation}
|\operatorname{supp}P(\cdot\mid s,a)|\le d_0,
\qquad (s,a)\in\Sset\times\Aset.
\end{equation}
The finite policy class $\Pi_{\mathcal O}\subseteq\Pi_{H,\eta}$ is fixed before sampling. Whenever the tightened problem is feasible, the certified planner returns an $\xi$-optimal policy in the sense of Definition~\ref{def:finite-certified-solver}.
\end{assumption}
}

The bounded-support condition restricts the number of possible successors. The planning requirement is an oracle condition: it controls statistical accuracy but does not assert that the constrained deterministic optimization is polynomial-time solvable. Algorithm~\ref{alg:three-stage} summarizes the procedure.

\begin{algorithm}[t]
\caption{Bellman-certified model-based learning}
\label{alg:three-stage}
\begin{algorithmic}[1]
\Require Generative model, finite certified solver, known support bound $d_0$, samples per row $n$,
tail allowance $\alpha_{\rm tail}$, grids $\eta_i$, margin $\rho$,
confidence $\zeta$, solver gap $\xi$

\State Compute $H_\alpha$, $b^0$, $w_{i,h}$, and $B_h$ as in
\eqref{eq:certificate-discretization} and Section~\ref{sec:certificates}.

\For{each original row $(s,a)$}
    \State Draw $n$ next-state samples, form $\widehat P_{s,a}$, and construct
    $\mathcal C_{s,a}$ by \eqref{eq:kl-confidence-set}.
\EndFor

\State Evaluate robust safety backups by \eqref{eq:pessimistic-backup}
and empirical reward backups by \eqref{eq:reward-backup}, reusing $\widehat P$.

\State Solve \eqref{eq:finite-certified-problem} over $\Pi_{\mathcal O}$
with $\delta_i$ replaced by {$\delta_i-3\rho/4$}, to gap $\xi$;
denote the output by $\widehat\pi$.

\If{{the tightened certified set is empty}}
    \State \Return {\textsc{unresolved}}
\Else
    \State \Return $\widehat\pi$
\EndIf
\end{algorithmic}
\end{algorithm}

\subsection{Sample Complexity Guarantees}
\label{sec:sample-complexity}
As is standard in finite-sample constrained RL, we state the guarantee relative to a comparator with a safety margin. Set $\bar\eta_\eps:=\min\{\eps/8,(1-\gamma)/128\}$, use $\alpha_{\rm tail}=\eta_i=\bar\eta_\eps$ in the rounded certificate, and let $q^{\rm rnd}_{P,i}(\pi)$ denote the resulting true rounded violation probability. For $\rho>0$, define the rounded $\rho$-interior optimal value
\begin{equation}
V_{\rho}^{\star}
=
\max\nolimits_{\pi\in\Pi_{\mathcal O}}
\{J_P(\pi): q^{\rm rnd}_{P,i}(\pi)\le\delta_i-\rho,\ i\in[m]\},
\label{eq:rounded-rho-comparator}
\end{equation}
and assume this set is nonempty. Let $\pi_{\rho}^{\star}$ be a maximizer. The margin is used only for statistical certification; conservative rounding separately guarantees safety for the original chance constraint. Such interior comparators are standard in finite-sample constrained RL \citep{Vaswani2022CMDPSampleComplexity,Buckley2025MultipleCMDP}.

\textcolor{black}{The output must satisfy the original constraints, but its return is compared with the best policy in $\Pi_{\mathcal O}$ that has margin $\rho$ under the rounded test. A policy can satisfy the original constraint but fail the rounded test if its cumulative cost often falls just below $d_i$. This can happen even with known transitions. Appendix~\ref{app:scalar-prob-proof} explains when a policy with margin $2\rho$ under the original constraints also has margin $\rho$ under the rounded test.}

We provide the upper and lower bounds below, with proofs in Appendix~\ref{app:scalar-prob-proof} and Appendix~\ref{app:model-based-lower-bounds}.

\begin{theorem}[Bellman-certified upper bound]
\label{thm:scalar-prob}
There are universal constants $C,c>0$ such that the following holds. Fix $\zeta\in(0,1)$, $\rho\in(0,1]$, $\eps\in(0,1)$, and $\xi\ge0$, and assume the set defining $V_{\rho}^{\star}$ is nonempty. Suppose Algorithm~\ref{alg:three-stage} uses a solver satisfying Assumption~\ref{ass:sharp-certified-planning}, with fixed support bound $d_0$. Run it with $\alpha_{\rm tail}=\eta_i=\bar\eta_\eps$ for every constraint and $n=\lfloor D/(|\Sset||\Aset|)\rfloor$ samples per row. If the total generative-model transition-sample budget $D$ satisfies
{
\begin{equation}
D
\ge
C d_0|\Sset||\Aset|
\left(
\frac{1}{(1-\gamma)^3\eps^2}
+
\frac{1}{(1-\gamma)\rho^2}
\right)
\log^c\!\left(
\frac{8d_0|\Sset||\Aset|}{\zeta\eps\rho(1-\gamma)}
\right)
.
\label{eq:scalar-total-samples}
\end{equation}
}
Then with probability at least $1-\zeta$, Algorithm~\ref{alg:three-stage} returns a policy $\widehat\pi$ satisfying $q_{P,i}(\widehat\pi)\le\delta_i$ for every $i\in[m]$ and $J_P(\widehat\pi)\ge V_{\rho}^{\star}-\eps-\xi$.
\end{theorem}

\begin{theorem}[Model-based CCMDP lower bound]
\label{thm:model-based-lower-bounds}
There are universal constants $c,c_0>0$ such that the following holds for every fixed $d_0\ge2$, $\gamma\in[1/2,1)$, $\eps\in(0,c_0/(1-\gamma))$, $\rho\in(0,c_0)$, {$|\Sset|\ge4$}, and {$|\Aset|\ge2$}. Use $\alpha_{\rm tail}=\eta=\bar\eta_\eps$ and take $\Pi_{\mathcal O}=\Pi_{H,\eta}$ in \eqref{eq:rounded-rho-comparator}. Let $\PiD$ be the full deterministic policy class represented by the trajectory parameters in Section~\ref{sec:problem}. Any generative-model learner that, on every one-constraint CCMDP satisfying $|\operatorname{supp}P(\cdot\mid s,a)|\le d_0$ and having a nonempty rounded $\rho$-interior feasible set, returns with probability at least $2/3$ a policy $\widehat\pi\in\PiD$ satisfying jointly $q_P(\widehat\pi)\le\delta$ and $J_P(\widehat\pi)\ge V^{\star}_\rho-\eps$ must, in the worst case, use a total number $D$ of transition samples satisfying
\begin{equation}
D
\ge
c|\Sset||\Aset|
\left(
\frac{1}{(1-\gamma)^3\eps^2}
+
{\frac{1}{(1-\gamma)\rho^2}}
\right).
\end{equation}
\end{theorem}

\paragraph{Sample complexity and comparison to prior CCMDP learning.}
\textcolor{black}{In Theorem~\ref{thm:scalar-prob}, the $\eps^{-2}$ term is the usual value-estimation cost, while the $\rho^{-2}$ term is the cost of certifying exact chance safety near the constraint boundary. The proof uses the shared KL sets to control the trajectory law without adding a confidence buffer at each Bellman step. A bound on total variation directly bounds the error in a violation probability. For discounted return, the same bound is multiplied by the return range, at most $1/(1-\gamma)$. The same transition samples must meet both accuracy requirements, so the larger of the two sample sizes suffices, or equivalently their sum up to a constant factor.}

\textcolor{black}{For fixed $d_0$, the upper bound improves the displayed state and horizon dependences of the deterministic CCMDP bound of \citep{Yi2025LCSS}. Theorem~\ref{thm:model-based-lower-bounds} shows that both the reward-learning term and the chance-boundary term are unavoidable, matching the upper bound in these parameters up to logarithmic factors.}

\paragraph{Comparison with MDPs and CMDPs.}
\textcolor{black}{The reward term has the same dependence as in discounted MDPs and relaxed-feasibility CMDPs~\citep{Azar2013minimaxPAC,agarwal2020model,Vaswani2022CMDPSampleComplexity}. Strict-feasibility CMDP bounds use a different feasibility margin, so their safety terms are not directly comparable.}

\section{Model-Free Stochastic Policy Learning}
\label{sec:model-free-tail-pg}
\textcolor{black}{We use sampled trajectories to optimize expected return subject to the rounded chance constraints, without estimating the transition kernel.}

\subsection{Trajectory Violation-Probability Policy Gradient}
\label{sec:model-free-algorithm}
\textcolor{black}{Compared with standard policy gradient for unconstrained MDPs, the chance-constrained setting requires three additional designs. First, the rollout must keep enough trajectory information to evaluate the chance event. Fix the certificate length \(H\) and budget grid from Section~\ref{sec:certificates}. During each rollout, the algorithm maintains the rounded remaining budget \(b_t\) and the clipped time index \(\bar h_t=\min\{t,H\}\), and uses the stochastic policy class $\Pi_\Theta=\{\pi_\theta(a\mid s,\bar h,b):\theta\in\ThetaSet\subset\mathbb R^p\}$, where $\ThetaSet$ is an open parameter domain and $(\bar h,b)$ are trajectory parameters rather than environment states.}
\textcolor{black}{From time $H$ onward, the policy uses $(\bar h_t,b_t)=(H,b_H)$. Only the violation event is truncated; the reward objective still uses the infinite-horizon return.}

\textcolor{black}{Second, we maximize expected return subject to constraints on the rounded violation probabilities. For constraint $i$, this probability is $q^H_{P,i}(\theta)
=
\mathbb P_{P,\pi_\theta,\mu}\!\left(b_{i,H}<0\right)$. This is the probability of the same finite event used by the Bellman distributional certificate.}

\textcolor{black}{Third, the optimization step must handle probability constraints estimated from trajectories. Algorithm~\ref{alg:trajectory-tail-pg} uses variance-reduced policy gradient~\citep{Yuan2020STORMPG,Zhang2021TSIVRPG} for the tightened problem. We estimate reward gradients from discounted rollouts, and chance-violation gradients from length-\(H\) violation indicators and likelihood-ratio scores.}

\textcolor{black}{Let $\bar J_P(\theta)=(1-\gamma)J_P(\pi_\theta)$, $f_H(\theta)=-\bar J_P(\theta)$, and $g_i^H(\theta)=q^H_{P,i}(\theta)+2\rho-\delta_i$. The $2\rho$ tightening leaves room for optimization and validation error. We write $g^H(\theta)\le0$ as $g^H(\theta)+z=0$ with $z\ge0$. The algorithm uses the smooth penalty $\Phi_\beta(\theta,z)=f_H(\theta)+(\beta/2)\|g^H(\theta)+z\|_2^2$ with fixed $\beta$. The proximal term $h$ is zero for $z\ge0$ and $+\infty$ otherwise, so the proximal step keeps the slack nonnegative.}

\textcolor{black}{We estimate the candidate's rounded violation probabilities on fresh trajectories. We accept it if each upper confidence bound is at most the allowed violation probability, and return \textsc{unresolved} otherwise. When the bounds hold, conservative rounding guarantees that the accepted policy satisfies the original constraints.}

\begin{algorithm}[ht]
\caption{Model-free trajectory-sampling violation-probability policy gradient}
\label{alg:trajectory-tail-pg}
\begin{algorithmic}[1]
\Require Rollout access, budget-aware policy class $\pi_\theta(a\mid s,\bar h,b)$, certificate length $H$, grids $\eta_i$, margin $\rho$, confidence $\zeta$, KKT tolerance $\eps_{\rm kkt}$
\State {Initialize $x_0=(\theta_0,z_0)$ using the feasible-initialization and merit-gap clause of Assumption~\ref{ass:finite-penalty}.}
\For{{$k=0,\ldots,T-1$}}
    \State Draw reward rollouts and length-$H$ violation-indicator rollouts, maintaining $(\bar h_t,b_t)$ {and sharing the full-budget rollout across constraints}.
    \State Update the recursive estimator $v_k$ of the smooth penalty gradient $\nabla\Phi_\beta(x_k)$, with periodic independent refresh batches.
    \State{Take the analyzed proximal penalty step $x_{k+1}=\operatorname{prox}_{\eta_{\rm opt}h}(x_k-\eta_{\rm opt}v_k)$, using independent samples for Jacobian--constraint products.}
\EndFor
\State {Draw $R$ uniformly from $\{0,\ldots,T-1\}$ and take the policy component $\theta_{\rm cand}$ of the already-computed $x_{R+1}$.} Validate it with fresh length-$H$ trajectories; accept only if $\widehat q^H_{i,\mathrm{val}}(\theta_{\rm cand})+\rho/2\le\delta_i$ for all $i$, and otherwise return \textsc{unresolved}.
\end{algorithmic}
\end{algorithm}

\subsection{Sample Complexity Guarantees}
\label{sec:model-free-guarantee}
\begin{assumption}[Local regularity]
\label{ass:stoch-regular}
\textcolor{black}{The objective and constraints are locally smooth. The policy scores, oracle variances, and multipliers are bounded. The parameter iterates stay in a compact subset of the open domain, and the initial policy is feasible for the tightened problem~\citep{Yuan2020STORMPG,Zhang2021TSIVRPG,Li2024StocIALM}. Assumption~\ref{ass:finite-penalty} gives the full conditions, including the additional residual-domination condition~\eqref{eq:residual-domination}.}
\end{assumption}

Define the KKT residual:
\begin{equation}
\mathcal R_H(\theta)
=
\min_{\lambda\in[0,\Lambda]^m}
\left\{
\left\|\nabla f_H(\theta)+\sum\nolimits_i\lambda_i\nabla g_i^H(\theta)\right\|_2
+\|[g^H(\theta)]_+\|_1
+\sum\nolimits_i|\lambda_i g_i^H(\theta)|
\right\},
\end{equation}
where \(\Lambda\) is the local multiplier bound from Assumption~\ref{ass:stoch-regular}.

\begin{theorem}[Model-free trajectory certificate]
\label{thm:tail-pg}
There is a universal constant $C_{\rm pg}>0$ such that the following holds. Fix $\alpha_{\rm tail}\in(0,(1-\gamma)^{-1})$, grid widths $\eta_i>0$, margin $\rho\in(0,1]$, KKT accuracy $\eps_{\rm kkt}\in(0,1)$, and confidence $\zeta\in(0,1)$. Run Algorithm~\ref{alg:trajectory-tail-pg} with \(H=\left\lceil \log(1/((1-\gamma)\alpha_{\rm tail}))/(1-\gamma)\right\rceil\), and use \(M_{\rm val}=\left\lceil 2\log(2m/\zeta)/\rho^2\right\rceil\) fresh validation trajectories per constraint. Let \(K_{\rm opt}\) denote the resulting problem-dependent local optimization constant. {Under Assumptions~\ref{ass:stoch-regular} and~\ref{ass:finite-penalty},} it suffices that the {expected total number $D:=\mathbb E[N_{\rm trans}]$} of environment transitions for optimization and validation satisfies
\begin{equation}
\begin{aligned}
D
\ge&
\frac{C_{\rm pg}K_{\rm opt}G^2m}{\eps_{\rm kkt}^{3}(1-\gamma)^3}
\log^3\left(\frac{2}{(1-\gamma)\alpha_{\rm tail}}\right)
\log^4\left(
\frac{8K_{\rm opt}G^2m}{\eps_{\rm kkt}\rho\zeta(1-\gamma)^3\alpha_{\rm tail}}
\right)
\\
&\quad+
\frac{C_{\rm pg}m}{\rho^2(1-\gamma)}
\log\left(\frac{2}{(1-\gamma)\alpha_{\rm tail}}\right)
\log\left(\frac{2m}{\zeta}\right).
\end{aligned}
\label{eq:variance-aware-pg-samples}
\end{equation}
Then the optimization phase returns a candidate satisfying
\(\mathbb E[\mathcal R_H(\theta_{\rm cand})]\le\eps_{\rm kkt}\).
Moreover, with probability at least \(1-\zeta\) over the independent validation trajectories, Algorithm~\ref{alg:trajectory-tail-pg} either returns validation unresolved or accepts a policy feasible for the original infinite-horizon CCMDP. On the same event, if {the returned candidate} satisfies \(q^H_{P,i}(\theta_{\rm cand})\le\delta_i-\rho\) for all \(i\), then it is accepted.

\end{theorem}
\paragraph{Discussion.}
\textcolor{black}{The first term in \eqref{eq:variance-aware-pg-samples} is the optimization cost, with local constants collected in $K_{\rm opt}$ (Appendix~\ref{app:proof-tail-pg}). The second is the cost of estimating the rounded violation probabilities on fresh trajectories.} \textcolor{black}{Bernoulli mean estimation with length-$H$ rollouts gives the validation cost $\widetilde O(m/((1-\gamma)\rho^2))$ in transition samples.} \textcolor{black}{The optimization guarantee is about stationarity: it neither bounds the gap to the global optimum nor guarantees that the candidate passes validation.}

\section{Numerical Studies}
\label{sec:experiments}

\begin{figure}[t]
    \centering
    \leavevmode\vbox{\offinterlineskip
        \hbox{\begin{subfigure}[t]{0.5\linewidth}
                \includegraphics[width=\linewidth]{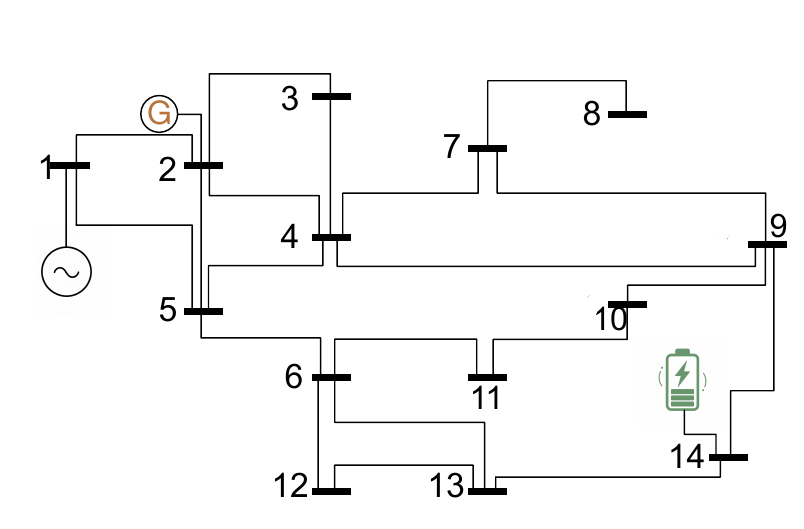}
                \caption{IEEE 14-bus Network}
                \label{fig:ieee-network}
            \end{subfigure}\begin{subfigure}[t]{0.5\linewidth}
                \includegraphics[width=\linewidth]{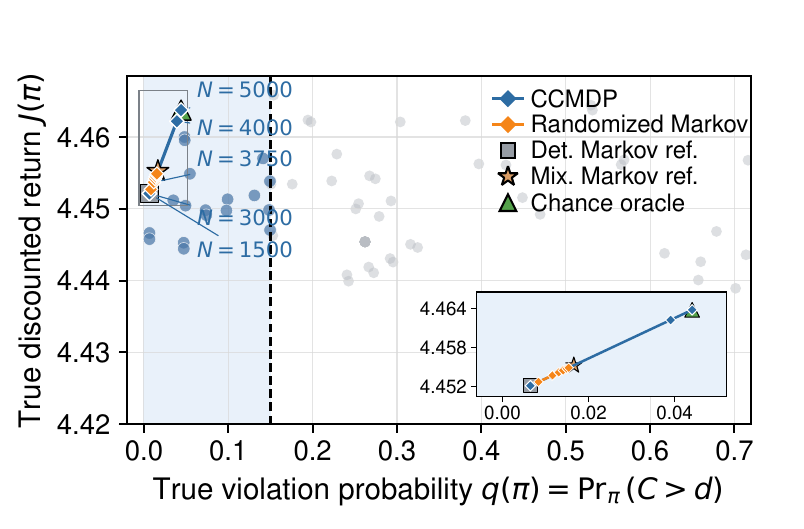}
                \caption{Policy landscape}
                \label{fig:ieee-landscape}
            \end{subfigure}}\hbox{\begin{subfigure}[t]{0.5\linewidth}
                \includegraphics[width=\linewidth]{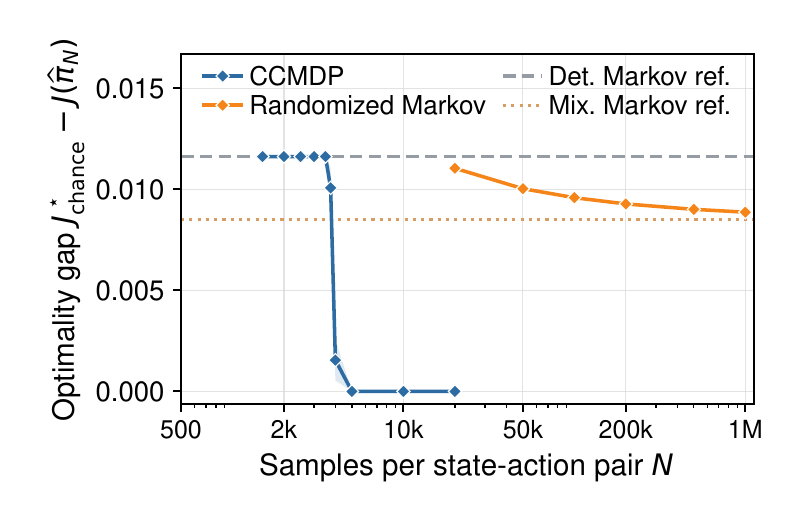}
                \caption{Optimality gap}
                \label{fig:ieee-gap}
            \end{subfigure}\begin{subfigure}[t]{0.5\linewidth}
                \includegraphics[width=\linewidth]{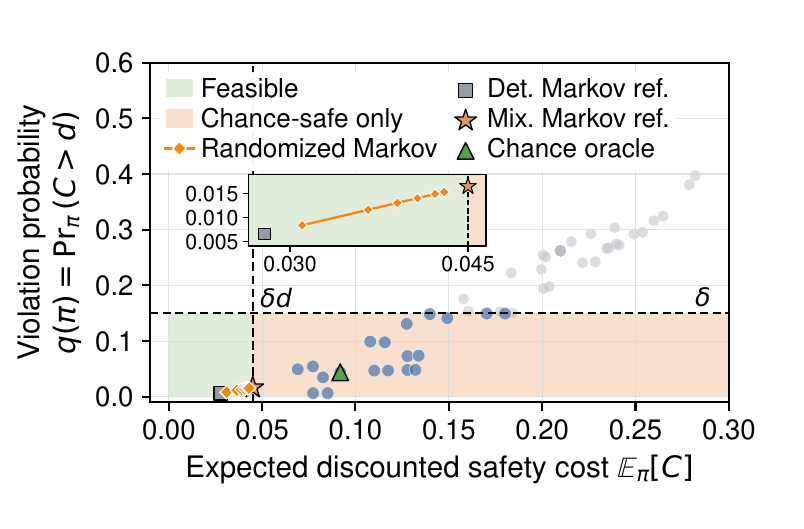}
                \caption{Expected-cost comparison}
                \label{fig:ieee-cost}
            \end{subfigure}}}\caption{{\color{black}
IEEE 14-bus storage control.
\textbf{(a)} DC network with the battery at bus~14.
\textbf{(b)} Return--violation landscape of the candidate policy class.
Blue and gray dots denote chance-feasible and chance-infeasible policies.
\textcolor{black}{Blue and orange diamonds show mean performance of CCMDP and
randomized Markov-CMDP, respectively. The indicated \(N\) values denote transition
samples per state-action pair.
The gray square, light-orange star, and green triangle mark the true-model
deterministic Markov reference, mixture Markov reference, and deterministic
chance oracle, respectively.}
\textbf{(c)} Optimality gap
\(J^\star_{\mathrm{chance}}-J(\widehat\pi_N)\) versus \(N\).
\textcolor{black}{The two curves and shaded bands} show the mean and \(\pm1.96\) empirical
standard errors from 60 independent empirical-model trials per sample
size, conditional on a policy being returned. \textcolor{black}{The gray dashed and
light-orange dotted lines show the deterministic and mixture Markov reference
gaps, respectively.}
\textbf{(d)} Violation probability versus expected discounted safety cost.
The horizontal and vertical dashed lines mark \(q_P(\pi)=\delta=0.15\)
and \(\mathbb E[C_\pi]=\delta d=0.045\), respectively.
The green region satisfies the Markov expected-cost condition; the orange
region contains chance-feasible policies excluded by that condition.
\textcolor{black}{Orange diamonds trace the learned mixtures.}
The reference markers are as in~(b).
}}
    \label{fig:ieee14-storage}
\end{figure}

\textcolor{black}{We test Algorithm~\ref{alg:three-stage} on a finite-state
CCMDP built from a DC power-flow simulator of the IEEE 14-bus
system~\citep{ieee14bus}.} The state records
storage state of charge, time block, and load regime; actions charge, discharge,
or idle a storage device placed at bus~14. Rewards measure normalized operating
benefit, while the safety cost measures normalized transmission-line overload
severity. The chance constraint limits the probability that the discounted
cumulative overload severity exceeds a prescribed threshold.
Appendix~\ref{app:ieee14-storage-setting} gives the complete experimental
specification.
\textcolor{black}{Additional synthetic diagnostics and model-free variance-reduced
policy-gradient (VR-PG) experiments, including optimization trajectories and
independent policy validation, are demonstrated in Appendix~\ref{app:synthetic-details}.}

{
\color{black}
We compare the Bellman-certified CCMDP selector in Algorithm~\ref{alg:three-stage} with a
\textcolor{black}{Markov-CMDP baseline using initial policy mixtures and a statistical
constraint margin~\citep{Vaswani2022CMDPSampleComplexity,Liu2025LinearCMDP}.} The
CCMDP selector reuses one empirical transition kernel for reverse-KL robust
safety backups and finite-horizon reward evaluation. \textcolor{black}{Markov-CMDP
maximizes the same empirical return over mixing weights whose weighted robust
expected-cost upper bound is at most \(\delta d=0.045\), using the same KL
confidence sets. One component policy is drawn at the start of each trajectory
and followed throughout.} The chance-constrained oracle
maximizes true return subject to \(q_P(\pi)\leq\delta\) within
\textcolor{black}{the original deterministic class}.
\par
\textcolor{black}{The curves in Figure~\ref{fig:ieee14-storage}(c) use only trials that return a policy. At $N=1500$, CCMDP returns a policy in 59 of 60 trials. The Markov-CMDP baseline returns no policy at any tested $N\le10{,}000$.}
\par
At smaller sample sizes, the policies returned by the algorithm coincide
with the \textcolor{black}{deterministic} Markov-CMDP reference, with return \(4.45213\). As the confidence sets
shrink, the selected policies approach the oracle (Figure~\ref{fig:ieee14-storage}).
At \(N=5000\) samples per state-action pair, all 60 trials select the oracle,
with return \(4.46376\). \textcolor{black}{Markov-CMDP's mean return increases from
\(4.45271\) at \(N=20{,}000\) to \(4.45489\) at \(N=10^6\), approaching the
true-model mixture reference \(4.45525\). This reference retains a gap of
\(0.00851\) to the deterministic chance oracle because of the expected-cost
constraint.}
\textcolor{black}{The experiment thus separates finite-sample conservatism of
the robust certificate from the conservatism of the expected-cost surrogate.}
\par
}

\section{Conclusion}
\label{sec:conclusion}
We introduced Bellman distributional certificates for learning infinite-horizon CCMDPs. The certificate turns chance safety into local violation-probability Bellman updates. We also proved a lower bound separating the reward-learning term from the chance-boundary certification term. Under the fixed bounded-successor-support and certified-planning conditions in Assumption~\ref{ass:sharp-certified-planning}, the model-based upper bound achieves exact safety and matches these primary dependences up to logarithmic factors. \textcolor{black}{We developed a model-free policy-gradient method with independent safety validation. Our experiments use a synthetic CCMDP and an IEEE 14-bus storage benchmark.}

\paragraph{Limitations and Future Directions.}
\label{sec:limitations}
The model-based guarantee is limited to deterministic policies, assumes a fixed bound on the number of successors of every state--action pair, and relies on the certified planning oracle in Assumption~\ref{ass:sharp-certified-planning}. The model-free guarantee applies to stochastic policies, but provides only {an approximate KKT-residual guarantee} and may return \textsc{unresolved} after validation. Extending the model-based analysis to broader policy classes, transition structures, and concrete solvers, and obtaining global model-free guarantees, are left for future work.

\bibliographystyle{plainnat}
\bibliography{ref.bib}

\newpage
\appendix
\setcounter{figure}{0}
\setcounter{table}{0}
\renewcommand{\thefigure}{A.\arabic{figure}}
\renewcommand{\thetable}{A.\arabic{table}}
\renewcommand{\theHfigure}{appendix.figure.\arabic{figure}}
\renewcommand{\theHtable}{appendix.table.\arabic{table}}

\section*{Roadmap for the appendix}
The Appendix is organized as follows:
\begin{itemize}
    \item Appendix~\ref{app:related} provides additional related work on CMDPs, CCMDPs, risk-sensitive RL, distributional RL, and post-selection certification.

    \item Appendix~\ref{app:certificate-parameter-sufficiency-proof} proves the Markov sufficiency of the rounded remaining-budget state and the Bellman recursion for finite-prefix violation probabilities.

    \item Appendix~\ref{app:tail-certification} proves the conservative rounding lemma that connects the finite rounded violation event to the original infinite-horizon chance constraint.

    \item Appendix~\ref{app:certified-safety-proof} establishes simultaneous KL row coverage and proves the certified-safety proposition.

    \item Appendix~\ref{app:scalar-prob-proof} proves the sample-complexity upper bound for Bellman-certified model-based learning.

    \item Appendix~\ref{app:model-based-lower-bounds} proves the model-based CCMDP lower bound.

    \item Appendix~\ref{app:deterministic-hardness} illustrates the computational hardness of exact deterministic planning for CCMDPs, showing that exact reward maximization over deterministic policies under a discounted chance constraint is $\mathsf{NP}$-hard.

    \item Appendix~\ref{app:proof-tail-pg} presents the model-free trajectory-certificate analysis, including the policy-gradient oracle identities and the proof of the model-free sample-complexity guarantee.

    \item Appendix~\ref{sec:exp_details} provides concise experimental settings for the synthetic diagnostics and the IEEE 14-bus storage-control experiment.
\end{itemize}

\section{Related Work}
\label{app:related}

Our work contributes to a few key literatures within the safe RL and theoretical RL community. We discuss each in turn below.

\paragraph{Safe RL and Constrained MDPs.}
Safe reinforcement learning studies sequential decision-making problems in which the learner must optimize reward while satisfying safety requirements~\citep{GarciaFernandez2015SafeRLSurvey,Wachi2024survey}. 
A large part of this literature builds on constrained Markov decision processes, where safety is modeled through constraints on expected cumulative costs~\citep{Derman1972ConstrainedChains,BeutlerRoss1985ConstrainedChains,Altman1999ConstrainedMDPs}. 
This formulation has enabled linear-programming and dynamic-programming characterizations in known models, as well as Lagrangian, primal-dual, policy-gradient, and Lyapunov-based methods for learning and policy optimization~\citep{Altman1999ConstrainedMDPs,Achiam2017CMDP_policy_opt,Chow2018Lyapunov,Paternain2019AdditiveProperty,Tessler2019RCPO,Ding2020NPGPDCMDP}. 
It has also supported finite-sample analyses for CMDP learning under exploration, generative-model, or primal-dual settings~\citep{efroni2020exploration,Vaswani2022CMDPSampleComplexity,Ding2021OPDOP,Buckley2025MultipleCMDP}. 
However, expected-cost constraints do not directly certify trajectory-level reliability. A policy may have small expected cost while still assigning non-negligible probability to unsafe trajectories, which motivates risk-sensitive, percentile, and chance-constrained formulations~\citep{haskell2013stochastic,borkar2014risk,Chow2018Percentile}. 
In contrast, our work focuses on chance constraints, where the safety requirement is imposed on the probability of violating a cumulative-cost threshold. This distinction changes both the algorithmic and statistical nature of the problem: the learner must control a tail-probability event rather than a first moment.

\paragraph{Chance-Constrained MDPs.}
Chance-constrained MDPs provide a direct way to model trajectory-level safety by requiring the probability of cumulative cost violation to stay below a prescribed threshold~\citep{Chow2018Percentile,Ono2015CCDP,Chen2024ProbabilisticConstraint,Shen2024FlippingCCMDP}, which are widely adopted in various practical scenarios \citep{DallAnese2017ChanceOPF,Zhao2018IntermodalChance,Beraldi2022EnhancedIndexation}. 
This formulation is substantially harder than expected-cost CMDPs: the constraint depends on the distribution of the cumulative cost around a threshold, the feasible set is generally nonconvex, and small model errors can change feasibility by moving probability mass across the threshold~\citep{haskell2013stochastic,borkar2014risk,delage2010percentile,Chow2018Percentile,Shen2024FlippingCCMDP,Yi2025LCSS}. 
Known-model approaches address these difficulties through budget augmentation, dynamic programming over cost distributions, or search-based planning in fully or partially observable models~\citep{Ono2015CCDP,Santana2016RAO,AlyassiKhonji2023CCSSP}. 
Learning and control approaches often use chance-constraint approximations, model-predictive control, actor-critic or Lagrangian updates, probabilistic-constraint gradients, or risk/percentile surrogates such as CVaR~\citep{Nemirovski2007ConvexApprox,Pfrommer2022safe,Peng2021MBCC,Peng2021CCRL_application,Giuseppi2020CCRL,chow2015risk,tamar2015optimizing,Chow2018Percentile,Chen2024ProbabilisticConstraint}. 
These methods are useful computationally or empirically, but they typically do not give finite-sample guarantees for the original chance event after a policy has been selected from data. The closest deterministic CCMDP learning result estimates policy-level violation probabilities for stationary deterministic policies~\citep{Yi2025LCSS}. 
Our work keeps the chance constraint itself as the certified object. On the model-based side, we use Bellman distributional certificates to certify the policy returned by a finite planner; on the model-free side, we optimize and validate the rounded violation probability directly from trajectories.

\paragraph{Finite-Sample Analysis for Model-Based Algorithms.}
Our model-based algorithm follows the standard model-based RL approach, where the learner estimates the transition model from a generative model and then plans in the empirical model. Model-based methods for discounted MDPs have been extensively studied~\citep{kearns2002sparse,Azar2013minimaxPAC,sidford2018near,agarwal2020model,li2020breaking}, achieving the minimax-optimal sample complexity \(\widetilde O(|\Sset||\Aset|\eps^{-2}(1-\gamma)^{-3})\) in the generative-model setting. Finite-sample guarantees have also been developed for CMDPs with expected cumulative cost constraints~\citep{efroni2020exploration,Ding2021OPDOP,Vaswani2022CMDPSampleComplexity,Buckley2025MultipleCMDP}; in this setting, the leading sample complexity has the same \(|\Sset||\Aset|\eps^{-2}(1-\gamma)^{-3}\) dependence, while exact constraint satisfaction typically requires a strict-feasibility margin. Chance-constrained MDPs are harder because feasibility is determined by a tail-probability event rather than an expected cumulative cost. {The closest deterministic CCMDP learning result estimates policy-level violation probabilities after model learning with sample complexity $\widetilde{\mathcal O}(|\Sset|^2|\Aset|(1-\gamma)^{-4}(\eps^{-2}+\rho^{-2}))$ when translated to our notation \citep{Yi2025LCSS}. In contrast, under fixed bounded successor support and access to a certified planning oracle, our upper bound has primary dependence $\widetilde O(|\Sset||\Aset|[(1-\gamma)^{-3}\eps^{-2}+(1-\gamma)^{-1}\rho^{-2}])$. The separate deterministic-policy lower bound has the same primary parameter dependences over the same bounded-successor problem class and uses the same rounded interior comparator.}

\paragraph{Finite-Sample Analysis for Model-Free Algorithms.}
Our model-free algorithm follows the policy-optimization view of RL, where the learner directly updates a policy from sampled trajectories without estimating the transition model~\citep{Sutton2018RLBook}. Model-free methods for safe RL and CMDPs have been widely studied through Lagrangian, primal-dual, natural policy-gradient, reward-constrained, Lyapunov-based, and backward-value-function approaches~\citep{Achiam2017CMDP_policy_opt,Tessler2019RCPO,Chow2018Lyapunov,WachiSui2020SafeRL,Satija2020BackwardValueFunctions,Ding2020NPGPDCMDP,Ding2021OPDOP}. These works provide important algorithmic and finite-sample guarantees for expected-cost constraints, but they do not directly control the probability of a trajectory-level violation event. A parallel line of work studies risk-sensitive or chance-constrained policy learning through CVaR objectives, chance-constraint approximations, model-predictive control, actor-critic updates, or probabilistic-constraint gradients~\citep{borkar2014risk,tamar2015optimizing,chow2015risk,Pfrommer2022safe,Peng2021MBCC,Giuseppi2020CCRL,Peng2021CCRL_application,Chen2024ProbabilisticConstraint}. These methods move closer to event-level safety, but their guarantees typically concern surrogate risks, approximations, asymptotic convergence, or local stationarity rather than finite-sample certification of the original chance constraint after learning. {To our knowledge, our result is the first to combine an explicit finite-sample approximate-KKT bound for the direct rounded chance objective with an independent post-learning certificate for the original chance event.} Technically, our algorithm combines a trajectory-level violation-probability gradient estimator with stochastic augmented-Lagrangian and variance-reduced policy-gradient ideas~\citep{Shen2019HAPG,Yuan2020STORMPG,Zhang2021TSIVRPG,Li2024StocIALM}, and then certifies the final data-dependent policy using an independent validation batch.

\paragraph{Distributional RL and Post-Selection Certification.}
Distributional RL learns the distribution of cumulative return and is often used to improve value estimation or optimize risk-sensitive criteria~\citep{bellemare2017distributional,dabney2018distributional,liang2024bridging}. Our use of ``distributional'' is different: we do not learn the full return distribution or replace the chance constraint with a surrogate risk measure. Instead, the certificate upper bounds the violation probability of the original threshold event. The main statistical issue is post-selection: the final policy is chosen after inspecting data, so fixed-policy validation is not sufficient. Our model-based algorithm addresses this by certifying Bellman violation-probability entries before policy selection, while our model-free algorithm uses an independent validation batch for the final learned policy.

\section{Proof of Lemma~\ref{lem:certificate-parameter-sufficiency}}
\label{app:certificate-parameter-sufficiency-proof}

\begin{proof}
Fix a constraint \(i\). Let \(\mathcal F_h\) be the history up to time \(h\), including \(s_h\) and \(b_h\). Define the future failure event
\begin{equation*}
E_{i,h}
=
\left\{\exists t\in\{h,\ldots,H\}: b_{i,t}<0\right\}.
\end{equation*}
We show by backward induction that its conditional probability depends only on \((s_h,b_h,h)\).

At \(h=H\), no transition remains. Thus
\begin{equation*}
\mathbb P_{P,\pi}(E_{i,H}\mid \mathcal F_H)
=
\ind\{b_{i,H}<0\}.
\end{equation*}
This is the terminal value of \(q^\pi_{P,i}\). Now let $h<H$ and assume the result holds at time \(h+1\). Given \(\mathcal F_h\), the action is $a_h=\pi_h(s_h,b_h)$. The next state follows $P$, and the next budget is
\begin{equation*}
b_{h+1}=B_h(s_h,a_h,b_h).
\end{equation*}
The event from time \(h\) is the event from time \(h+1\) after this update. The induction hypothesis and the tower property give
\begin{equation*}
\begin{aligned}
\mathbb P_{P,\pi}(E_{i,h}\mid\mathcal F_h)
&=
\sum\nolimits_{s'}P(s'\mid s_h,a_h)\,
q^\pi_{P,i}(s',B_h(s_h,a_h,b_h),h+1)  =
q^\pi_{P,i}(s_h,b_h,h),
\end{aligned}
\end{equation*}
where the last equality is the Bellman recursion. The last expression uses the history only through \((s_h,b_h,h)\). This proves the induction step and the lemma.
\end{proof}

\section{Conservative Rounding}
\label{app:tail-certification}

\begin{lemma}[Conservative rounding]
\label{lem:conservative-rounding}
For every policy $\pi\in\Pi_{H,\eta}$ and constraint $i\in[m]$, the original infinite-horizon violation probability is bounded by the rounded finite-prefix violation probability:
\begin{equation}
q_{P,i}(\pi)\le q^{\rm rnd}_{P,i}(\pi).
\end{equation}
\end{lemma}
\begin{proof}
Set $H=H_\alpha$. Fix a trajectory and a constraint $i$. Suppose the rounded test does not fail by time $H$. Then
\begin{equation*}
\sum\nolimits_{h=0}^{H-1}w_{i,h}(s_h,a_h)\le b_i^0.
\end{equation*}
Each rounded charge is an upper bound, and $b_i^0$ is rounded down. Hence
\begin{equation*}
\sum\nolimits_{h=0}^{H-1}\gamma^h c_i(s_h,a_h)
\le
\eta_i\sum\nolimits_{h=0}^{H-1}w_{i,h}(s_h,a_h)
\le
\eta_i b_i^0
\le
d_i-\alpha_{\rm tail}.
\end{equation*}
The cost after time $H$ is at most
\begin{equation*}
\sum\nolimits_{h=H}^{\infty}\gamma^h c_i(s_h,a_h)
\le
\sum\nolimits_{h=H}^{\infty}\gamma^h
=
\frac{\gamma^H}{1-\gamma}
\le
\alpha_{\rm tail}.
\end{equation*}
The two bounds give $\sum_{h=0}^{\infty}\gamma^h c_i(s_h,a_h)\le d_i$. Thus, if the original constraint fails, the rounded test must also fail. Taking probabilities gives $q_{P,i}(\pi)\le q^{\rm rnd}_{P,i}(\pi)$.
\end{proof}

\section{Proof of Proposition~\ref{prop:certified-safety}}
\label{app:certified-safety-proof}

We first show that every true transition row belongs to its KL confidence set and then apply Bellman monotonicity.

\begin{lemma}[Simultaneous KL row coverage]
\label{lem:kl-row-coverage}
Under Assumption~\ref{ass:sharp-certified-planning}, the confidence sets in \eqref{eq:kl-confidence-set} satisfy
\begin{equation}
\mathbb P\!\left(
P(\cdot\mid s,a)\in\mathcal C_{s,a}
\text{ for every }(s,a)
\right)
\ge1-\zeta.
\end{equation}
\end{lemma}
For one row whose true support has size at most $d_0$, the method-of-types bound gives
\(
\mathbb P(\mathrm{KL}(\widehat P_{s,a}\|P_{s,a})>x)
\le(n+1)^{d_0-1}e^{-nx}
\); see, e.g., \citep{Agrawal2020MultinomialKL}. Substituting $x=\kappa_n$ and taking a union bound over the $|\Sset||\Aset|$ original rows proves the lemma. Notice that the unknown identities of the successors introduce no additional factor because every empirical type is supported on the true row support.

Fix any realized dataset in the event of Lemma~\ref{lem:kl-row-coverage}, any policy $\pi\in\Pi_{H,\eta}$, and any constraint $i$. We show that the robust Bellman table dominates the true rounded table.

We use backward induction. At $h=H$, both tables equal $\ind\{b_i<0\}$. Suppose the bound holds at $h+1$. For $a=\pi_h(s,b)$,
\begin{equation*}
\begin{aligned}
P(\cdot\mid s,a)^\top q_{P,i}^{\pi}(\cdot,B_h(s,a,b),h+1)
&\le P(\cdot\mid s,a)^\top\overline q_i^\pi(\cdot,B_h(s,a,b),h+1)\\
&\le \sup_{p\in\mathcal C_{s,a}}
p^\top\overline q_i^\pi(\cdot,B_h(s,a,b),h+1).
\end{aligned}
\end{equation*}
The first inequality uses the induction hypothesis, and the second uses $P_{s,a}\in\mathcal C_{s,a}$. Induction and averaging over $s_0\sim\mu$ give $q^{\rm rnd}_{P,i}(\pi)\le\overline q_{i,H,\eta}(\pi)$. If the certificate accepts $\pi$, then
\begin{equation*}
q_{P,i}(\pi)
\le q^{\rm rnd}_{P,i}(\pi)
\le\overline q_{i,H,\eta}(\pi)
\le\delta_i.
\end{equation*}
The argument is deterministic on the row-coverage event and therefore holds simultaneously for every policy and constraint. This proves the proposition.\hfill\(\square\)

\section{Proof of Theorem~\ref{thm:scalar-prob}}
\label{app:scalar-prob-proof}

The proof has four steps. We record the finite rounded problem, establish a uniform trajectory-transfer bound, prove safety, and then bound reward loss and sample size.

Let
\[
H=H_\alpha
=
\left\lceil
\frac{\log(1/((1-\gamma)\alpha_{\rm tail}))}{1-\gamma}
\right\rceil,
\qquad
R_H=\sum_{h=0}^{H-1}\gamma^h,
\qquad
|\mathcal B|=\prod_{i=1}^m N_i(\alpha_{\rm tail},\eta_i),
\]
Here $N_i(\alpha_{\rm tail},\eta_i)$ is the number of values in budget coordinate $i$, including failure. Fix a maximizer $\pi_\rho^\star$ of $V_\rho^\star$. Let $J_{P,H}(\pi):=\mathbb E_{P,\pi,\mu}[\sum_{h=0}^{H-1}\gamma^h r(S_h,A_h)]$ and set $r_n:=\sqrt{H\kappa_n/2}$.
Under the theorem's choice \(\alpha_{\rm tail}=\eta_i=\bar\eta_\eps\), we have \(H=\widetilde O((1-\gamma)^{-1})\), \(R_H\le(1-\gamma)^{-1}\), and \(\gamma^H/(1-\gamma)\le\bar\eta_\eps\le\eps/8\).

\par\medskip
\noindent\textbf{Step 1: Record the rounded state space.}\par
\nopagebreak[4]
Coordinate $i$ has at most \(\max\{1,\lfloor(d_i-\bar\eta_\eps)/\bar\eta_\eps\rfloor+2\}\) values. This finiteness makes the certified planning problem well defined, but $|\mathcal B|$ does not enter the confidence event because the same original-row KL set is valid for all time and budget states. The planner may still take time proportional to $|\mathcal B|$.

\par\medskip
\noindent\textbf{Step 2: Transfer the shared empirical model uniformly.}\par
\nopagebreak[4]
Let $\widehat q_i^\pi$ be the rounded violation probability under the shared empirical kernel $\widehat P$. \textcolor{black}{On the row-coverage event of Lemma~\ref{lem:kl-row-coverage}, the following bound holds for every policy in the class, including the policy selected from the data.}
\begin{lemma}[Uniform trajectory transfer]
\label{lem:uniform-kl-transfer}
On the event of Lemma~\ref{lem:kl-row-coverage}, every $\pi\in\Pi_{H,\eta}$ and $i\in[m]$ satisfy
\begin{equation}
\begin{aligned}
|\widehat q_i^\pi-q^{\rm rnd}_{P,i}(\pi)|&\le r_n,
&0\le\overline q_{i,H,\eta}(\pi)-q^{\rm rnd}_{P,i}(\pi)&\le2r_n,\\
|\widehat J(\pi)-J_{P,H}(\pi)|&\le R_Hr_n.
\end{aligned}
\end{equation}
\end{lemma}

The sample condition in the theorem permits an integer $n=D/(|\Sset||\Aset|)$, up to rounding, such that $r_n\le\rho/8$ and $R_Hr_n\le\eps/8$. Applying Lemma~\ref{lem:uniform-kl-transfer} gives, simultaneously for $\pi_\rho^\star$ and the data-dependent output $\widehat\pi$,
\[
\left|\widehat J(\pi)-J_{P,H}(\pi)\right|\le\frac{\eps}{8},
\qquad
0\le
\overline q_{i,H,\eta}(\pi)-q^{\rm rnd}_{P,i}(\pi)
\le\frac{\rho}{4}.
\]

\par\medskip
\noindent\textbf{Step 3: Prove feasibility and safety.}\par
\nopagebreak[4]

The rounded comparator satisfies $q^{\rm rnd}_{P,i}(\pi_\rho^\star)\le\delta_i-\rho$. The uniform transfer bound therefore gives
\[
\overline q_{i,H,\eta}(\pi_\rho^\star)
\le
\delta_i-\rho+\frac{\rho}{4}
\le
\delta_i-\frac{3\rho}{4},
\qquad i\in[m].
\]
Thus $\pi_\rho^\star$ is feasible for the tightened empirical problem, so Algorithm~\ref{alg:three-stage} does not return \textsc{unresolved} on the row-coverage event. Its robust feasibility and Lemma~\ref{lem:uniform-kl-transfer} give, for each $i$,
\[
q^{\rm rnd}_{P,i}(\widehat\pi)
\le
\overline q_{i,H,\eta}(\widehat\pi)
\le
\delta_i-\frac{3\rho}{4}
\le
\delta_i.
\]
Lemma~\ref{lem:conservative-rounding} now gives $q_{P,i}(\widehat\pi)\le\delta_i$.

\par\medskip
\noindent\textbf{Step 4: Bound reward loss and count samples.}\par
\nopagebreak[4]
Now compare rewards. Nonnegative rewards give $J_P(\widehat\pi)\ge J_{P,H}(\widehat\pi)$. The reward after time \(H\) is at most \(\gamma^H/(1-\gamma)\le\eps/8\). Also, \(\pi_\rho^\star\) is feasible for the tightened problem, so empirical $\xi$-optimality applies. Hence
\[
\begin{aligned}
J_P(\widehat\pi)
&\ge J_{P,H}(\widehat\pi)\ge \widehat J(\widehat\pi)-\frac{\eps}{8}\ge \widehat J(\pi_\rho^\star)-\xi-\frac{\eps}{8}\ge J_{P,H}(\pi_\rho^\star)-\xi-\frac{\eps}{4}\\
&\ge J_P(\pi_\rho^\star)-\xi-\frac{3\eps}{8}
\ge V_\rho^\star-\xi-\eps .
\end{aligned}
\]
The next-to-last inequality uses the reward-tail bound. The last one uses $J_P(\pi_\rho^\star)=V_\rho^\star$ and $3\eps/8\le\eps$.

Finally, $r_n\le\min\{\rho/8,\eps/(8R_H)\}$ follows from $n=\widetilde O(d_0H[\rho^{-2}+R_H^2\eps^{-2}])$. The algorithm draws this batch once at each original state--action row, so $D=|\Sset||\Aset|n$. Substituting $H=\widetilde O((1-\gamma)^{-1})$ and $R_H\le(1-\gamma)^{-1}$ gives \eqref{eq:scalar-total-samples}.\hfill\(\square\)

\subsection{Proof of Lemma~\ref{lem:uniform-kl-transfer}}
Fix a realized dataset in the event of Lemma~\ref{lem:kl-row-coverage}. For a policy $\pi$, let $\mathbb P_{M,\pi}^{0:H}$ denote the law of the augmented trajectory through time $H$ under transition kernel $M$. The KL chain rule and \eqref{eq:kl-confidence-set} give
\[
\mathrm{KL}\!\left(
\mathbb P_{\widehat P,\pi}^{0:H}
\,\middle\|\,
\mathbb P_{P,\pi}^{0:H}
\right)
=
\sum_{h=0}^{H-1}
\mathbb E_{\widehat P,\pi}\!\left[
\mathrm{KL}(\widehat P_{S_h,A_h}\|P_{S_h,A_h})
\right]
\le H\kappa_n.
\]
Pinsker's inequality bounds the total variation distance between these trajectory laws by $r_n$. Applying it to the rounded failure indicator gives $|\widehat q_i^\pi-q^{\rm rnd}_{P,i}(\pi)|\le r_n$. Applying it to the truncated discounted reward, whose range is contained in $[0,R_H]$, gives $|\widehat J(\pi)-J_{P,H}(\pi)|\le R_Hr_n$.

It remains to control the conservatism of the robust certificate. Rectangularity of the row-wise sets implies that the backward suprema in \eqref{eq:pessimistic-backup} are attained by a possibly nonstationary augmented-state kernel $Q_i^\pi$ satisfying $Q_{i,h}^\pi(\cdot\mid s,b)\in\mathcal C_{s,\pi_h(s,b)}$ at every triple $(s,b,h)$. Its rounded failure probability equals $\overline q_{i,H,\eta}(\pi)$. A second application of the KL chain rule yields
\[
\mathrm{KL}\!\left(
\mathbb P_{\widehat P,\pi}^{0:H}
\,\middle\|\,
\mathbb P_{Q_i^\pi,\pi}^{0:H}
\right)
\le H\kappa_n.
\]
Because $\widehat P_{s,a}\in\mathcal C_{s,a}$, robust maximization and Pinsker give $0\le\overline q_{i,H,\eta}(\pi)-\widehat q_i^\pi\le r_n$. Because $P_{s,a}\in\mathcal C_{s,a}$, Proposition~\ref{prop:certified-safety} gives $q^{\rm rnd}_{P,i}(\pi)\le\overline q_{i,H,\eta}(\pi)$. Combining these inequalities with $|\widehat q_i^\pi-q^{\rm rnd}_{P,i}(\pi)|\le r_n$ proves the safety claim.

All arguments are deterministic after conditioning on the row event and use no policy union bound. Hence they hold simultaneously for every $\pi\in\Pi_{H,\eta}$, including a data-dependent output policy.\hfill\(\square\)

{
\paragraph{Relation to the original CCMDP optimum.}
The rounded finite event is a conservative certificate, not a different safety notion. Lemma~\ref{lem:conservative-rounding} gives \(q_{P,i}(\pi)\le q^{\rm rnd}_{P,i}(\pi)\), so any policy certified by the rounded event is safe for the original infinite-horizon chance constraint. If $C_i^\pi:=\sum_{t\ge0}\gamma^tc_i(S_t,A_t)$ and $\Delta_i:=\alpha_{\rm tail}+H_\alpha\eta_i$, the same rounding inequalities imply
\[
q_{P,i}(\pi)
\le q^{\rm rnd}_{P,i}(\pi)
\le
\mathbb P_{P,\pi,\mu}\!\left(C_i^\pi>d_i-\Delta_i\right).
\]
Thus a true $2\rho$-interior comparator also belongs to the rounded $\rho$-interior set whenever its probability mass on $(d_i-\Delta_i,d_i]$ is at most $\rho$ for every constraint. Consequently, if $V_{2\rho}^{\star,\mathrm{true}}:=\max\{J_P(\pi):q_{P,i}(\pi)\le\delta_i-2\rho,\ i\in[m]\}$ has such a maximizer, Theorem~\ref{thm:scalar-prob} also guarantees reward at least $V_{2\rho}^{\star,\mathrm{true}}-\eps-\xi$. Without this boundary condition, the rounded and original comparators need not agree, even when the transition kernel is deterministic.
}

\section{Proof of Theorem~\ref{thm:model-based-lower-bounds}}
\label{app:model-based-lower-bounds}

The proof has three steps. We prove the reward lower bound, prove the safety lower bound, and add the two bounds. Every transition row in both hard families has at most two successors, so the constructions satisfy Assumption~\ref{ass:sharp-certified-planning} for every fixed $d_0\ge2$.

\par\medskip
\noindent\textbf{Step 1: Prove the reward lower bound.}\par
\nopagebreak[4]
Set $H_0=\lceil(1-\gamma)^{-1}\rceil$, $M=|\Sset|-1$, and $B=|\Aset|$. Build $M$ decision states $x_1,\ldots,x_M$ and one absorbing state $z$. The initial state is uniform over the decision states. The reward is one at each $x_j$ and zero at $z$. Set $c\equiv0$, $d=1$, and $\delta=1/2$. Then every policy is safe under both the original and rounded events. For $c_0<1/2$, the rounded $\rho$-interior optimum is the usual reward optimum.

Let $\vartheta=(\vartheta_1,\ldots,\vartheta_M)\in[B]^M$ be hidden. At $x_j$, action $\vartheta_j$ is good. The process returns to $x_j$ with probability
\begin{equation*}
p_s=1-\frac{1}{4H_0}
\quad\text{if }a=\vartheta_j,
\qquad
p_u=p_s-\Delta
\quad\text{otherwise},
\qquad
\Delta=\frac{2048\eps}{H_0^2};
\end{equation*}
and moves to $z$ otherwise. Rewards and costs are known. Thus only transition samples contain information about $\vartheta$. Choose $c_0$ small enough that $\Delta\le1/(4H_0)$. Then $p_u\ge1-1/(2H_0)$, so all probabilities are valid.

Fix any deterministic output policy $\pi$; it may depend on history. Consider the history that starts at $x_j$ and stays there for $t$ transitions. Let $a_{j,t}$ be the action of $\pi$ on this history. Let $S_j(t)$ be the probability of staying at $x_j$ until time $t$, with $S_j(0)=1$. Let $W_j=\sum_{t=0}^{H_0-1}\ind\{a_{j,t}\ne\vartheta_j\}$. Also let $V^\pi(x_j)$ be the value from $x_j$.

The all-good policy has value $V^\star=(1-\gamma p_s)^{-1}$. Since $V^\star=1+\gamma p_sV^\star$, each bad action loses $\gamma\Delta V^\star$. Thus the performance-difference identity gives
\begin{equation*}
V^\star-V^\pi(x_j)
=
\gamma\Delta V^\star
\sum_{t\ge0}\gamma^tS_j(t)
\ind\{a_{j,t}\ne\vartheta_j\}.
\end{equation*}
For $t<H_0$, we have $V^\star\ge H_0/2$, $\gamma^t\ge1/4$, and $S_j(t)\ge p_u^t\ge1/2$. Average over the initial state to get
\begin{equation*}
V^\star-J_P(\pi)
\ge
\frac{\Delta H_0}{32M}\sum_{j=1}^M W_j.
\end{equation*}
If $\pi$ is $\eps$-optimal, the choice of $\Delta$ gives $\sum_jW_j\le MH_0/64$. Decode $\vartheta_j$ as the most common action in $(a_{j,0},\ldots,a_{j,H_0-1})$. Use fixed tie-breaking. A wrong answer requires $W_j\ge H_0/2$. Thus the decoder is correct on at least $31M/32$ coordinates.

A sample from a decision row is Bernoulli with mean $p_s$ or $p_u$. The Bernoulli KL bound gives
\begin{equation*}
\kappa
:=
\max\!\left\{
\mathrm{kl}(\mathrm{Bern}(p_s),\mathrm{Bern}(p_u)),
\mathrm{kl}(\mathrm{Bern}(p_u),\mathrm{Bern}(p_s))
\right\}
\le C H_0\Delta^2
\le C\frac{\eps^2}{H_0^3}.
\end{equation*}
Choose $c_0$ small enough for Lemma~\ref{lem:adaptive-needle}. Draw $\vartheta$ uniformly from $[B]^M$, and draw a hidden index $J$ uniformly from $[M]$. The learner succeeds with probability at least $2/3$. The decoder above then identifies $\vartheta_J$ with probability at least $(2/3)(31/32)>3/5$.

Now draw the hidden actions for all contexts except $J$ from the prior and simulate their samples internally. Send only queries at $J$ to the one-good-row experiment. This gives the same product-prior law. The expected number of sent queries is at most $D/M$. Lemma~\ref{lem:adaptive-needle} gives $D/M\ge cB/\kappa$. Since $M\ge|\Sset|/2$ and $H_0\ge(1-\gamma)^{-1}$,
\begin{equation*}
D
\ge
c\frac{MB}{\kappa}
\ge
c'\frac{|\Sset||\Aset|}{(1-\gamma)^3\eps^2},
\end{equation*}
This is the reward lower bound.

\par\medskip
\noindent\textbf{Step 2: Prove the safety lower bound.}\par
\nopagebreak[4]
The next lemma gives the safety lower bound. We prove it in Appendix~\ref{app:chance-boundary-proof}.
\begin{lemma}[Chance-boundary testing lower bound]
\label{lem:chance-boundary-lower}
There are universal constants $c,c_0>0$ such that, for every $\gamma\in[1/2,1)$, $\rho\in(0,c_0)$, $|\Sset|\ge4$, and $|\Aset|\ge2$, there exists a zero-reward, one-constraint CCMDP family with two successors per row and a nonempty rounded $\rho$-interior feasible set for which any generative-model learner that returns an exactly safe policy in $\PiD$ with probability at least $2/3$ must use, in the worst case,
\begin{equation}
D
\ge
c\,\frac{|\Sset||\Aset|}{(1-\gamma)\rho^2}
\end{equation}
transition samples.
\end{lemma}

In the family from Lemma~\ref{lem:chance-boundary-lower}, all rewards are zero and the rounded $\rho$-interior set is nonempty. Thus exact safety alone requires $D\ge c|\Sset||\Aset|/((1-\gamma)\rho^2)$.

\par\medskip
\noindent\textbf{Step 3: Add the two bounds.}\par
\nopagebreak[4]
Take the union of the two hard families. Its worst-case complexity is at least the larger bound. Since $\max\{x,y\}\ge(x+y)/2$, this gives the sum up to a universal constant.
\hfill\(\square\)

\subsection{Proof of Lemma~\ref{lem:chance-boundary-lower}}
\label{app:chance-boundary-proof}
The proof has four steps. We build the hard family, decode its safe actions, bound the KL of one sample, and apply a testing lower bound.

\par\medskip
\noindent\textbf{Step 1: Build the hard family.}\par
\nopagebreak[4]
Set $H_0=\lceil(1-\gamma)^{-1}\rceil$, $M=|\Sset|-2$, and $B=|\Aset|$. Build decision states $x_1,\ldots,x_M$, a failure state $f$, and a safe absorbing state $z$. The initial state is uniform over $x_1,\ldots,x_M$. Let $\vartheta\in[B]^M$ be hidden. Action $\vartheta_j$ is the safe action at $x_j$. Set
\[
p_s=1-\frac{1}{4H_0},
\qquad
\Delta=\frac{128\rho}{H_0},
\qquad
p_u=p_s-\Delta.
\]
At $x_j$, the process stays at $x_j$ with probability $p_s$ under action $\vartheta_j$ and with probability $p_u$ under any other action. It moves to $f$ otherwise. From $f$, it moves to $z$ in one step. State $z$ is absorbing.

All rewards are zero. The cost is one at $f$ and zero elsewhere. Set
\[
d=\frac{\gamma^{H_0}+\gamma^{H_0+1}}{2},
\qquad
q_s=1-p_s^{H_0},
\qquad
\delta=q_s+\rho.
\]
Bernoulli's inequality gives $p_s^{H_0}\ge3/4$. Choose $c_0$ small enough that $0<\delta<1$.

The state $f$ is visited for one step. If it is reached at time $\tau$, the total cost is $\gamma^\tau$. Since $\gamma^{H_0+1}<d<\gamma^{H_0}$, violation means reaching $f$ by time $H_0$. The all-safe policy has violation probability $q_s=\delta-\rho$. Thus the $\rho$-interior set is nonempty. Its value is zero because all rewards are zero.

We verify that this is also the rounded event used in the theorem. Here each trajectory pays at most one nonzero cost. Moreover,
\[
\frac{\gamma^{H_0}-\gamma^{H_0+1}}{2}
=
\frac{\gamma^{H_0}(1-\gamma)}{2}
\ge\frac{1-\gamma}{16},
\qquad
\alpha_{\rm tail}+2\eta
\le\frac{3(1-\gamma)}{128}.
\]
Also, $\alpha_{\rm tail}\le(1-\gamma)/128$ implies
$H_\alpha\ge\log(128/(1-\gamma)^2)/(1-\gamma)>H_0+1$. Therefore a payment at time $\tau\le H_0$ crosses the rounded threshold, while a payment at time $\tau\ge H_0+1$ does not: the latter follows from $\gamma^{H_0+1}+\alpha_{\rm tail}+2\eta<d$. Hence the original and rounded violation indicators agree trajectory by trajectory throughout this hard family, and the all-safe policy is a rounded $\rho$-interior comparator.

\par\medskip
\noindent\textbf{Step 2: Recover the safe actions.}\par
\nopagebreak[4]
Fix a deterministic output policy $\pi$; it may depend on history. At $x_j$, consider the history that stays at $x_j$ for the first $t$ transitions. Let $a^\pi_{j,t}$ be the action on this history. This defines an action sequence for $t=0,\ldots,H_0-1$. Let $W_j=\sum_{t=0}^{H_0-1}\ind\{a^\pi_{j,t}\ne\vartheta_j\}$. Starting from $x_j$, the probability of violation is
\[
q_j(\pi)
=
1-\prod_{t=0}^{H_0-1}
\left(p_s-\Delta\ind\{a^\pi_{j,t}\ne\vartheta_j\}\right).
\]
Changing one factor from $p_s$ to $p_u$ lowers survival by at least $\Delta p_u^{H_0-1}$. A telescoping expansion gives
\[
q_j(\pi)-q_s
\ge
\Delta p_u^{H_0-1}W_j,
\qquad
q_P(\pi)-q_s
\ge
\frac{\Delta p_u^{H_0-1}}{M}\sum_{j=1}^M W_j.
\]
Choose $c_0$ so that $128\rho\le1/4$. Then $p_u\ge1-1/(2H_0)$, $p_u^{H_0-1}\ge1/2$, and $H_0\Delta p_u^{H_0-1}\ge64\rho$. If $\pi$ is safe, $q_P(\pi)\le q_s+\rho$. Hence $\sum_jW_j\le MH_0/64$.

Decode $\vartheta_j$ as the most common action in $(a^\pi_{j,0},\ldots,a^\pi_{j,H_0-1})$. Use fixed tie-breaking. A wrong answer needs $W_j\ge H_0/2$. The decoder can therefore be wrong on at most $M/32$ states. Every safe policy recovers at least $31M/32$ entries of $\vartheta$.

\par\medskip
\noindent\textbf{Step 3: Bound the KL of one sample.}\par
\nopagebreak[4]
A decision-row sample is Bernoulli with mean $p_s$ or $p_u$. The mean gap is $\Delta=O(\rho/H_0)$, and the failure probability is of order $H_0^{-1}$. The Bernoulli KL bound gives
\[
\kappa
:=
\max\!\left\{
\mathrm{kl}(\mathrm{Bern}(p_s),\mathrm{Bern}(p_u)),
\mathrm{kl}(\mathrm{Bern}(p_u),\mathrm{Bern}(p_s))
\right\}
\le C H_0\Delta^2
\le C\frac{\rho^2}{H_0}.
\]
Thus one sample has only $O(\rho^2/H_0)$ KL. Choose $c_0$ small enough for the lemma used next.

\par\medskip
\noindent\textbf{Step 4: Reduce to one random state.}\par
\nopagebreak[4]
We use the following one-context testing bound. Its proof is in Appendix~\ref{app:adaptive-needle-proof}.
\begin{lemma}[Adaptive one-good-row identification]
\label{lem:adaptive-needle}
Let $B\ge2$, let $\Theta$ be uniform on $[B]$, and suppose that a query to arm $a$ returns an independent Bernoulli observation with mean $p_s\in(0,1)$ when $a=\Theta$ and mean $p_u\in(0,1)$ otherwise. Define $\kappa=\max\{\mathrm{kl}(\mathrm{Bern}(p_s),\mathrm{Bern}(p_u)),\mathrm{kl}(\mathrm{Bern}(p_u),\mathrm{Bern}(p_s))\}$. For a sufficiently small universal $\kappa_0>0$, every adaptive randomized procedure with $0<\kappa\le\kappa_0$ that identifies $\Theta$ with probability at least $3/5$, averaged over the uniform prior, uses an expected number $N$ of queries satisfying
\begin{equation}
\mathbb E[N]\ge c\,\frac{B}{\kappa}
\end{equation}
for a universal constant $c>0$.
The conclusion is unchanged in the presence of auxiliary observations whose conditional law given the past is the same for every value of $\Theta$.
\end{lemma}

Draw $\vartheta$ uniformly from $[B]^M$. Also draw a hidden index $J$ uniformly from $[M]$. A safe output is correct on at least a $31/32$ fraction of states. The learner returns such an output with probability at least $2/3$. Hence $\mathbb P(\widehat\vartheta_J=\vartheta_J)\ge(2/3)(31/32)=31/48>3/5$.

Simulate every state except $J$ internally. Send only queries at $J$ to the one-good-row oracle. This produces the same law as the original hard family, and the learner does not know $J$. If $N_j$ is the number of samples from state $j$, then
\[
\mathbb E[N_J]
=
\frac1M\sum_{j=1}^M\mathbb E[N_j\mid J=j]
=
\frac1M\sum_{j=1}^M\mathbb E_{\rm prior}[N_j]
\le
\frac{D}{M}.
\]
Lemma~\ref{lem:adaptive-needle} gives $D/M\ge cB/\kappa\ge cBH_0/\rho^2$. Use $M\ge|\Sset|/2$, $B=|\Aset|$, and $H_0\ge(1-\gamma)^{-1}$ to get the claimed bound.
\hfill\(\square\)

\subsection{Proof of Lemma~\ref{lem:adaptive-needle}}
\label{app:adaptive-needle-proof}
The proof has three steps. We truncate the run, compare it with a null model, and apply Pinsker's inequality.

\par\medskip
\noindent\textbf{Step 1: Stop the procedure.}\par
\nopagebreak[4]
Let $n=\mathbb E[N]$ under the uniform prior. Stop after $\lceil32n\rceil$ queries. By Markov's inequality, this changes the output with probability at most $1/32$. The stopped procedure therefore succeeds with probability at least $3/5-1/32$.

\par\medskip
\noindent\textbf{Step 2: Define a null model.}\par
\nopagebreak[4]
Let $\mathbb Q_1$ be the law of the hidden arm and the stopped transcript. The transcript includes the algorithm's random bits and output. Define a null law $\mathbb Q_0$ in which every arm has mean $p_u$ and $\Theta$ is still uniform. Under $\mathbb Q_0$, the transcript is independent of $\Theta$. Thus any estimate is correct with probability $1/B$. If $N_a$ is the number of queries to arm $a$, the KL chain rule gives
\[
\mathrm{KL}(\mathbb Q_0,\mathbb Q_1)
=
\frac1B\sum_{a=1}^B
\mathbb E_0[N_a]\,
\mathrm{kl}(\mathrm{Bern}(p_u),\mathrm{Bern}(p_s))
\le
\frac{\kappa\lceil32n\rceil}{B}.
\]
Given $\Theta=a$, only samples from arm $a$ differ under the two laws. Auxiliary observations have the same law, so they add zero KL.

\par\medskip
\noindent\textbf{Step 3: Apply Pinsker's inequality.}\par
\nopagebreak[4]
Since $1/B\le1/2$, the two success probabilities differ by at least $(3/5-1/32)-1/2=11/160$. Pinsker's inequality gives $\kappa\lceil32n\rceil/B\ge c_1$. For small enough $\kappa_0$, this implies $n\ge cB/\kappa$.
\hfill\(\square\)

\section{Computational Hardness of Exact Deterministic Planning}
\label{app:deterministic-hardness}

\begin{proposition}[Hardness of exact deterministic planning]
\label{prop:deterministic-hardness}
Even when the transition model is known and deterministic, exact reward maximization over deterministic policies subject to a discounted chance constraint is $\mathsf{NP-hard}$.
\end{proposition}

\begin{proof}
Reduce from 0--1 knapsack. Let the item values be $v_j$, the weights be $w_j$, and the capacity be $B$. Fix a rational $\gamma\in(0,1)$. Build a deterministic chain through states $1,\ldots,n$. State $n+1$ is absorbing. At state $j$, action \emph{choose} gives reward $v_j/(M_v\gamma^{j-1})$ and cost $w_j/(M_w\gamma^{j-1})$. Action \emph{skip} gives zero reward and cost. Both actions move to state $j+1$. Set
\[
M_v=\gamma^{-(n-1)}\max_j v_j,
\qquad
M_w=\gamma^{-(n-1)}\max_j w_j,
\]
Replace a zero maximum by one. Then all one-step rewards and costs lie in $[0,1]$.

A deterministic policy chooses a vector $x\in\{0,1\}^n$. The factors $\gamma^{j-1}$ cancel, so
\[
J(x)=\frac{1}{M_v}\sum_{j=1}^n v_jx_j,
\qquad
C(x)=\frac{1}{M_w}\sum_{j=1}^n w_jx_j.
\]
Set $d=B/M_w$ and take any $\delta\in(0,1)$. The trajectory is deterministic, so its violation probability is zero or one. The chance constraint holds exactly when $C(x)\le d$, or $\sum_jw_jx_j\le B$.

Thus feasible policies are exactly feasible knapsack solutions. Their discounted reward is proportional to total item value. An exact planner would solve 0--1 knapsack. The construction has polynomial size for fixed rational $\gamma$, so the problem is $\mathsf{NP}$-hard.
\end{proof}

\section{Proof of Theorem~\ref{thm:tail-pg}}
\label{app:proof-tail-pg}
The proof has three steps. We define the rounded problem, find a KKT candidate, and test this candidate on fresh samples.

\par\medskip
\noindent\textbf{Step 1: Define the rounded problem and its oracles.}\par
\nopagebreak[4]
Set $H$ as in Theorem~\ref{thm:tail-pg}. A rollout starts from $b_0=b^0$ and updates the rounded budget until time $H$. After time $H$, set $b_t=b_H$ and $\bar h_t=H$, so $a_t\sim\pi_\theta(\cdot\mid s_t,\bar h_t,b_t)$ with $\bar h_t=\min\{t,H\}$. For constraint $i$, let $q^H_{P,i}(\theta)=\mathbb P_{P,\pi_\theta,\mu}(b_{i,H}<0)$ and define $f_H(\theta)=-\bar J_P(\theta)$ and $g_i^H(\theta)=q^H_{P,i}(\theta)+2\rho-\delta_i$.
The rounded problem is
\begin{equation*}
\min_{\theta\in\ThetaSet} f_H(\theta)
\qquad\text{s.t.}\quad
g_i^H(\theta)\le0,\quad i\in[m].
\end{equation*}
All iterates stay inside $\mathcal K\Subset\ThetaSet$. Thus the KKT residual has no boundary term for $\theta$. Assumption~\ref{ass:stoch-regular} gives $\|\nabla_\theta\log\pi_\theta(a\mid s,\bar h,b)\|_2\le G$.
The next two lemmas give the stochastic gradients used by the optimizer.
\begin{lemma}[Trajectory first-order oracle]
\label{lem:trajectory-oracle}
Under Assumption~\ref{ass:stoch-regular}, the rollouts defined below give unbiased estimates of $f_H(\theta)$, $\nabla f_H(\theta)$, $g_i^H(\theta)$, and $\nabla g_i^H(\theta)$. Their gradient second moments satisfy
\begin{equation}
\mathbb E\|\widehat{\nabla f_H}(\theta)\|_2^2\le \frac{2G^2}{(1-\gamma)^2},
\qquad
\mathbb E\sum\nolimits_{i=1}^m\|\widehat{\nabla g_i^H}(\theta)\|_2^2\le mG^2H^2,
\end{equation}
and their value samples satisfy
\begin{equation}
|\widehat f_H(\theta)|\le1,\qquad
|\widehat g_i^H(\theta)|\le3 .
\end{equation}
One call uses at most $(1-\gamma)^{-1}+H$ transitions in expectation. Its transition-weighted variance is therefore
\begin{equation}
\mathsf W_H
=
O\!\left(
G^2\left[\frac{1}{(1-\gamma)^3}+mH^3\right]\right).
\end{equation}
\end{lemma}

The recursive update also needs an unbiased gradient difference.
\begin{lemma}[Trajectory gradient differences]
\label{lem:trajectory-difference}
For each $F\in\{g_1^H,\ldots,g_m^H\}$, sample $\tau$ under $\pi_\theta$ and define
\begin{equation}
L_{\theta'\mid\theta}(\tau)
=
\frac{p_{\theta'}(\tau)}{p_\theta(\tau)}
\end{equation}
and $\Delta_F(\theta,\theta';\tau)=\widehat\nabla F(\theta;\tau)-L_{\theta'\mid\theta}(\tau)\widehat\nabla F(\theta';\tau)$. Then
\begin{equation}
\mathbb E_{\tau\sim p_\theta}\Delta_F(\theta,\theta';\tau)
=
\nabla F(\theta)-\nabla F(\theta'),
\qquad
\mathbb E_{\tau\sim p_\theta}\|\Delta_F(\theta,\theta';\tau)\|_2^2
\le
L_{\rm pg}^2\|\theta-\theta'\|_2^2 .
\end{equation}
\textcolor{black}{For the exact normalized reward with geometric stopping, the same statements hold under the mean-square condition for likelihood-ratio differences in item~3 of Assumption~\ref{ass:finite-penalty}.} They also hold for a truncated reward rollout, although Theorem~\ref{thm:tail-pg} does not use that approximation.
\end{lemma}

Lemmas~\ref{lem:trajectory-oracle} and~\ref{lem:trajectory-difference} give the required refresh and difference estimators, with transition-weighted variance $\mathsf W_H=O\!\left(G^2[(1-\gamma)^{-3}+mH^3]\right)$.

\par\medskip
\noindent\textbf{Step 2: Find a KKT candidate.}\par
\nopagebreak[4]
For $\min_\theta f(\theta)$ subject to $g_i(\theta)\le0$, define
\begin{equation*}
\mathcal R(\theta)
=
\min_{\lambda\in[0,\Lambda]^m}
\left\{
\left\|\nabla f(\theta)+\sum_i\lambda_i\nabla g_i(\theta)\right\|_2
+\|[g(\theta)]_+\|_1
+\sum_i|\lambda_i g_i(\theta)|
\right\}.
\end{equation*}
The next lemma bounds this residual.
\begin{lemma}[Finite violation-probability KKT theorem]
\label{lem:vr-alm-template}
Under Assumption~\ref{ass:finite-penalty}, the recursive variance-reduced method returns $\widehat\theta$ with $\mathbb E[\mathcal R(\widehat\theta)]\le\eps$ after
\begin{equation}
\widetilde O\!\left(K_{\rm opt}\mathsf V\eps^{-3}\right)
\end{equation}
stochastic first-order oracle calls, where $\mathsf V$ and $K_{\rm opt}$ are defined in that assumption. If oracle components have different rollout lengths, the expected number of transitions is $\widetilde O(K_{\rm opt}\mathsf W\eps^{-3})$.
\end{lemma}

Apply Lemma~\ref{lem:vr-alm-template} to $f_H$ and $g_i^H$. Assumptions~\ref{ass:stoch-regular} and~\ref{ass:finite-penalty} give all needed conditions. Their constants are in $K_{\rm opt}$ and do not depend on $\eps_{\rm kkt}$ or $\zeta$. The candidate satisfies $\mathbb E[\mathcal R_H(\theta_{\rm cand})]\le\eps_{\rm kkt}$.
Let \(L_\alpha=\log(2/((1-\gamma)\alpha_{\rm tail}))\). Then
\begin{equation*}
H=
\left\lceil
\frac{\log(1/((1-\gamma)\alpha_{\rm tail}))}{1-\gamma}
\right\rceil
\le
\frac{C L_\alpha}{1-\gamma},
\end{equation*}
for a universal constant \(C\). Insert this bound and \(\mathsf W_H\) into Lemma~\ref{lem:vr-alm-template}. The optimization cost is
\begin{equation*}
\frac{C_{\rm pg}K_{\rm opt}G^2m}{\eps_{\rm kkt}^{3}(1-\gamma)^3}
L_\alpha^3
\log^4\!\left(
\frac{8K_{\rm opt}G^2m}{\eps_{\rm kkt}\rho\zeta(1-\gamma)^3\alpha_{\rm tail}}
\right)
\end{equation*}
expected transitions. This is the first term of \eqref{eq:variance-aware-pg-samples}.

\par\medskip
\noindent\textbf{Step 3: Validate safety.}\par
\nopagebreak[4]
Now validate the candidate. Condition on the optimization data. Then $\theta_{\rm cand}$ is fixed, the validation trajectories are i.i.d., and $\widehat q^H_{i,\mathrm{val}}(\theta_{\rm cand})=M_{\rm val}^{-1}\sum_{j=1}^{M_{\rm val}}\ind\{b_{i,H}^{(j)}<0\}$ is a Bernoulli average with mean $q^H_{P,i}(\theta_{\rm cand})$. Hoeffding's inequality gives
\begin{equation*}
\mathbb P\!\left(
\left|
\widehat q^H_{i,\mathrm{val}}(\theta_{\rm cand})
-q^H_{P,i}(\theta_{\rm cand})
\right|
>
\sqrt{\frac{\log(2m/\zeta)}{2M_{\rm val}}}
\;\middle|\;\theta_{\rm cand}
\right)
\le \frac{\zeta}{m}.
\end{equation*}
A union bound over $i$ and the choice of $M_{\rm val}$ give, with probability at least $1-\zeta$,
\begin{equation*}
q^H_{P,i}(\theta_{\rm cand})
\le
\widehat q^H_{i,\mathrm{val}}(\theta_{\rm cand})
+\sqrt{\frac{\log(2m/\zeta)}{2M_{\rm val}}}
\le
\widehat q^H_{i,\mathrm{val}}(\theta_{\rm cand})+\rho/2
\end{equation*}
for every $i$. The next lemma links the rounded event to the original event.
\begin{lemma}[Rounded finite tail event certifies infinite-horizon safety]
\label{lem:stoch-tail}
Fix $H=\lceil\log(1/((1-\gamma)\alpha_{\rm tail}))/(1-\gamma)\rceil$ and the conservative rounded budget process used by Algorithm~\ref{alg:trajectory-tail-pg}. For any budget-aware stochastic policy $\pi_\theta$,
\begin{equation}
\left\{\sum\nolimits_{t=0}^{\infty}\gamma^tc_i(s_t,a_t)>d_i\right\}
\subseteq
\left\{b_{i,H}<0\right\}.
\label{eq:stoch-tail-inclusion}
\end{equation}
Consequently,
\begin{equation}
q_{P,i}(\theta)\le q^H_{P,i}(\theta).
\label{eq:stoch-tail-probability}
\end{equation}
\end{lemma}

Suppose validation passes and set $\widehat\theta=\theta_{\rm cand}$. Then $q^H_{P,i}(\widehat\theta)\le\delta_i$. Lemma~\ref{lem:stoch-tail} gives $q_{P,i}(\widehat\theta)\le\delta_i$. Thus every accepted policy is safe. If the test fails, the algorithm returns \textsc{unresolved}.

Validation uses $mM_{\rm val}H$ transitions. The bound on $H$ gives the second term of \eqref{eq:variance-aware-pg-samples}.

It remains to prove acceptance under a margin. Suppose $q^H_{P,i}(\theta_{\rm cand})\le\delta_i-\rho$ for every $i$. On the same event,
\begin{equation*}
\widehat q^H_{i,\mathrm{val}}(\theta_{\rm cand})
\le q^H_{P,i}(\theta_{\rm cand})+\rho/2
\le \delta_i-\rho/2 .
\end{equation*}
Thus $\widehat q^H_{i,\mathrm{val}}(\theta_{\rm cand})+\rho/2\le\delta_i$ for all $i$. The test accepts the candidate. This proves Theorem~\ref{thm:tail-pg}.

\subsection{Proof of Lemma~\ref{lem:trajectory-oracle}}
\begin{proof}
We handle reward and constraints separately.

\par\medskip
\noindent\textbf{Step 1: Construct the reward oracle.}\par
\nopagebreak[4]
Let $\bar J_P=(1-\gamma)J_P$. Draw $\tau$ with $\mathbb P(\tau=t)=(1-\gamma)\gamma^t$. Run the policy through time $\tau$ and set
\begin{equation*}
\widehat{\bar J}(\theta)=r(s_\tau,a_\tau),\qquad
\widehat{\nabla\bar J}(\theta)
=
r(s_\tau,a_\tau)
\sum\nolimits_{u=0}^{\tau}\nabla_\theta\log\pi_\theta(a_u\mid s_u,\bar h_u,b_u).
\end{equation*}
Then $\mathbb E[\widehat{\bar J}(\theta)]=(1-\gamma)\sum_{t\ge0}\gamma^t\mathbb E[r(s_t,a_t)]=\bar J_P(\theta)$. The score identity gives
\begin{equation*}
\nabla\bar J_P(\theta)
=
(1-\gamma)\sum\nolimits_{t\ge0}\gamma^t\nabla_\theta\mathbb E[r(s_t,a_t)]
=
\mathbb E[\widehat{\nabla\bar J}(\theta)],
\end{equation*}
Rewards and scores are bounded, so the derivative may pass through the sum. Moreover, $\|\widehat{\nabla\bar J}(\theta)\|_2^2\le G^2(\tau+1)^2$. Using $\mathbb E[(\tau+1)^2]\le2(1-\gamma)^{-2}$ and $\mathbb E[\tau+1]=(1-\gamma)^{-1}$ proves the reward bounds. For $f_H=-\bar J_P$, only the sign changes, and $|\widehat f_H|\le1$.

\par\medskip
\noindent\textbf{Step 2: Construct the constraint oracle.}\par
\nopagebreak[4]
One length-$H$ rollout gives all constraint samples. Define
\begin{equation*}
\chi_i=\ind\{b_{i,H}<0\},\qquad
\widehat g_i^H(\theta)=\chi_i+2\rho-\delta_i,\qquad
\widehat{\nabla g_i^H}(\theta)
=
\chi_i\sum\nolimits_{t=0}^{H-1}\nabla_\theta\log\pi_\theta(a_t\mid s_t,t,b_t).
\end{equation*}
Since $\mathbb E\chi_i=q^H_{P,i}(\theta)$, the value sample is unbiased. For the gradient, fix a path $\omega=(s_0,a_0,\ldots,s_H)$. Its budgets and $\chi_i(\omega)$ are fixed. Only the policy terms depend on $\theta$:
\begin{equation*}
\begin{aligned}
\mathbb P_\theta(\omega)
&=
\mu(s_0)\prod\nolimits_{t=0}^{H-1}\pi_\theta(a_t\mid s_t,t,b_t)P(s_{t+1}\mid s_t,a_t),\\
\nabla_\theta\mathbb P_\theta(\omega)
&=
\mathbb P_\theta(\omega)
\sum\nolimits_{t=0}^{H-1}\nabla_\theta\log\pi_\theta(a_t\mid s_t,t,b_t).
\end{aligned}
\end{equation*}
Sum over all paths. This gives $\mathbb E[\widehat{\nabla g_i^H}(\theta)]=\nabla q^H_{P,i}(\theta)=\nabla g_i^H(\theta)$.

Finally, $\|\widehat{\nabla g_i^H}(\theta)\|_2^2\le G^2H^2$ and $|\widehat g_i^H(\theta)|\le3$. A constraint rollout therefore gives $O(mG^2H^3)$ transition-weighted variance. The reward rollout gives $O(G^2(1-\gamma)^{-3})$. Adding them gives $\mathsf W_H$.
\end{proof}

\subsection{Proof of Lemma~\ref{lem:trajectory-difference}}
\begin{proof}
The proof has two steps: unbiasedness and mean-square smoothness.

\par\medskip
\noindent\textbf{Step 1: Prove unbiasedness.}\par
\nopagebreak[4]
The path fixes the rounded budgets. Thus its law depends on $\theta$ only through the policy, with $p_\theta(\tau)=\mu(s_0)\prod_{t=0}^{H-1}\pi_\theta(a_t\mid s_t,t,b_t)P(s_{t+1}\mid s_t,a_t)$.
Policy positivity makes $L_{\theta'\mid\theta}=p_{\theta'}/p_\theta$ well defined. Hence
\begin{equation*}
\mathbb E_{\tau\sim p_\theta}
\!\left[L_{\theta'\mid\theta}(\tau)\widehat\nabla F(\theta';\tau)\right]
=
\mathbb E_{\tau\sim p_{\theta'}}[\widehat\nabla F(\theta';\tau)]
=
\nabla F(\theta').
\end{equation*}
Lemma~\ref{lem:trajectory-oracle} gives the other term. Thus $\Delta_F(\theta,\theta';\tau)$ is unbiased for $\nabla F(\theta)-\nabla F(\theta')$.

\par\medskip
\noindent\textbf{Step 2: Establish mean-square smoothness.}\par
\nopagebreak[4]
Let $R_F(u;\tau)=L_{u\mid\theta}(\tau)\widehat\nabla F(u;\tau)$. The mean-value theorem gives
\begin{equation*}
\|\Delta_F(\theta,\theta';\tau)\|_2
\le
\sup\nolimits_{u\in[\theta,\theta']}\|\nabla_u R_F(u;\tau)\|_{\rm op}
\|\theta-\theta'\|_2 .
\end{equation*}
For a violation component, $\widehat\nabla F(u;\tau)=\chi(\tau)\sum_{t=0}^{H-1}\nabla\log\pi_u(a_t\mid s_t,t,b_t)$. Thus $\nabla_uR_F$ contains likelihood ratios, scores, and score derivatives. Their second moments are bounded on the compact set. Call this bound $L_{\rm pg}^2$. This proves the second-moment claim.

\textcolor{black}{For the reward oracle, the change of measure gives unbiasedness conditional on the geometric stopping time, which has the same law under both policies. We use item~3 of Assumption~\ref{ass:finite-penalty} for the mean-square bound on this difference estimator. For truncated reward rollouts, the finite-horizon argument above applies.}
\end{proof}

\subsection{Local KKT conditions}
The fixed-penalty problem uses slack variables. The next lemma links its residuals to the KKT residual.
\begin{lemma}[Slack residual implies inequality KKT residual]
\label{lem:slack-residual}
Consider $\min_\theta f(\theta)$ subject to $g_i(\theta)\le0$, $i\in[m]$. For a slack variable $z\in[0,Z]^m$ and multiplier $y\in\mathbb R^m$, define
\begin{equation}
r_\theta=\left\|\nabla f(\theta)+J_g(\theta)^\top y\right\|_2,\quad
r_c=\|g(\theta)+z\|_1,\quad
r_z=\operatorname{dist}\!\left(0,y+N_{\mathbb R_+^m}(z)\right).
\end{equation}
Here $N_{\mathbb R_+^m}(z)$ is the normal cone. If $\|J_g(\theta)\|_{\mathrm{op}}\le B_J$ and $\|[y]_+\|_1\le\Lambda$, then $\lambda=[y]_+$ satisfies
\begin{equation}
\mathcal R(\theta)
\le
r_\theta+B_Jr_z+r_c+\Lambda r_c+Z\sqrt m\,r_z .
\end{equation}
\end{lemma}

\begin{assumption}[Finite violation-probability local regularity]
\label{ass:finite-penalty}
Together with Assumption~\ref{ass:stoch-regular}, assume the following local conditions. In particular, item 4 is an additional residual-domination condition. Consider $\min_\theta f(\theta)$ subject to $g_i(\theta)\le0$, $i\in[m]$, on an open domain $\ThetaSet$. Every parameter iterate and every segment between consecutive iterates lies in a compact convex set $\mathcal K\Subset\ThetaSet$. Set $x=(\theta,z)$, $z\in\mathbb R_+^m$, $h(x)=\ind_{\mathbb R_+^m}(z)$, and $c(x)=g(\theta)+z$. Thus only the slack variable is constrained by $h$. Fix $\beta\ge1$, independently of the target accuracy, and let $\Phi_\beta(x)=f(\theta)+(\beta/2)\|c(x)\|_2^2$ and $F_\beta=\Phi_\beta+h$. All constants below are finite and independent of the target accuracy:
\begin{enumerate}
\item The initialization $x_0=(\theta_0,z_0)$ has $g(\theta_0)\le0$, $z_0\ge0$, and $F_\beta(x_0)-\inf_xF_\beta(x)\le\Delta$.
\item The slack error bound and the following bounds hold:
\begin{equation}
\begin{aligned}
\nu\|c(x)\|_2
&\le
\operatorname{dist}\!\left(-J_c(x)^\top c(x),\partial h(x)\right),\\
\|\nabla f(\theta)\|_2&\le B_f,\qquad
\|J_g(\theta)\|_{\mathrm{op}}\le B_J,\qquad
\|z\|_\infty\le Z .
\end{aligned}
\end{equation}
Here $\partial h(x)=\{0\}\times N_{\mathbb R_+^m}(z)$.
\item The functions $f$ and $g_i$ have Lipschitz gradients, and $y=\beta c(x)$ satisfies $\|[y]_+\|_1\le\Lambda$ along the visited iterates. The stochastic samples are unbiased; refresh estimators have bounded variance; and recursive differences are mean-square smooth. For fixed-distribution oracles, use common-random-number differences. For trajectory gradients, use the likelihood-ratio difference estimator from Lemma~\ref{lem:trajectory-difference}. With $C(x,\xi)=\widehat g(\theta,\xi)+z$ and $J_C(x,\xi)=[\,\widehat J_g(\theta,\xi)\ \ I_m\,]$, there are constants $K_c,K_J,K_L,K_V$ such that
\begin{equation}
\begin{aligned}
&\mathbb E\|C(x,\xi)\|_2^2\le K_c^2,\qquad
\mathbb E\|J_C(x,\xi)\|_{\mathrm{op}}^2\le K_J^2,\\
&\mathbb E\|C(x,\xi)-C(x',\xi)\|_2^2
\le K_L^2\|x-x'\|_2^2,\\
&\mathbb E\|J_C(x,\xi)-J_C(x',\xi)\|_{\mathrm{op}}^2
\le K_L^2\|x-x'\|_2^2,\\
&\mathbb E\|\widehat\nabla f(\theta;\xi)-\widehat\nabla f(\theta';\xi)\|_2^2
\le K_L^2\|x-x'\|_2^2,\\
&\mathbb E\|\widehat\nabla f(\theta;\xi)-\nabla f(\theta)\|_2^2
\le K_V\mathsf V,\\
&\mathbb E\|C(x,\xi)-c(x)\|_2^2
+\mathbb E\|J_C(x,\xi)-J_c(x)\|_{\mathrm{op}}^2
\le K_V\mathsf V,\quad \mathsf V\ge1 .
\end{aligned}
\end{equation}
\item For every visited $x$, let $y=\beta c(x)$ and define $r_\theta,r_z,r_c$ as in Lemma~\ref{lem:slack-residual}. Assume
\begin{equation}
r_\theta(x)+r_z(x)+r_c(x)
\le
K_{\rm res}\,
\operatorname{dist}\!\left(0,\nabla\Phi_\beta(x)+\partial h(x)\right)
\label{eq:residual-domination}
\end{equation}
for a finite constant $K_{\rm res}$.
\end{enumerate}
We write
\begin{equation}
K_{\rm opt}
=
\operatorname{poly}\!\left(
L_f,L_g,B_J,\Lambda,Z,\nu^{-1},\Delta,
K_c,K_J,K_L,K_V,L_{\rm pg},K_{\rm res},\beta
\right)
\end{equation}
Here $L_f,L_g$ are the Lipschitz-gradient constants and $L_{\rm pg}$ is the trajectory-difference constant. These quantities may depend on the policy class, horizon, budget grid, tightening margin, and visited set, but not on $\eps$ or the confidence level. We absorb universal polynomial products of them into $K_{\rm opt}$.
\end{assumption}

This assumption is used only for the KKT bound. The safety proof uses fresh validation data.

\subsection{Proof of Lemma~\ref{lem:vr-alm-template}}
\textcolor{black}{Let $x^+=(\theta^+,z^+)$ be the computed iterate $x_{R+1}$ selected in Algorithm~\ref{alg:trajectory-tail-pg}.} We first bound the KKT residual by the stationarity residual of $F_\beta=\Phi_\beta+h$. Assumption~\ref{ass:finite-penalty} gives
\begin{equation*}
r_\theta(x^+)+r_z(x^+)+r_c(x^+)
\le
K_{\rm opt}\,
\operatorname{dist}\!\left(0,\nabla\Phi_\beta(x^+)+\partial h(x^+)\right).
\end{equation*}
Lemma~\ref{lem:slack-residual} shows that it is enough to make
\begin{equation*}
\operatorname{dist}\!\left(0,\nabla\Phi_\beta(x^+)+\partial h(x^+)\right)
\le c\,\eps/K_{\rm opt}
\end{equation*}
for a small universal constant $c$. The next lemma gives an oracle for $F_\beta$.
\begin{lemma}[Penalty-gradient oracle calibration]
\label{lem:penalty-oracle}
Under Assumption~\ref{ass:finite-penalty}, for the fixed $\beta\ge1$, the smooth part
\begin{equation*}
\Phi_\beta(x)=f(\theta)+\frac{\beta}{2}\|c(x)\|_2^2
\end{equation*}
has $K_{\rm sm}\beta$-Lipschitz gradient, an unbiased refresh-gradient oracle with variance at most $K_\sigma^2\beta^2\mathsf V$, and a recursive gradient-difference oracle with mean-square constant $K_{\rm ms}\beta$. Moreover, $F_\beta=\Phi_\beta+h$ satisfies $F_\beta(x_0)-\inf_xF_\beta(x)\le\Delta$.
\end{lemma}

Thus, after increasing $K_{\rm opt}$, we have $L_\beta\le K_{\rm opt}$, $\sigma_\beta^2\le K_{\rm opt}\mathsf V$, and $L_{{\rm ms},\beta}\le K_{\rm opt}$. Let $\bar\sigma_\beta=\max\{1,\sigma_\beta\}$. We use the next proximal bound.
\begin{lemma}[Self-contained variance-reduced proximal subproblem]
\label{lem:self-contained-vr-subproblem}
Let $F(x)=\Phi(x)+h(x)$, where $\Phi$ is $L$-smooth, $h$ is closed convex with a tractable proximal map, and $F(x_0)-\inf_xF(x)\le\Delta$. Suppose the refresh estimator is unbiased with variance at most $\sigma^2$ and the gradient-difference estimator is unbiased with second moment at most $L_{\rm ms}^2\|x-x'\|_2^2$. Set $\bar\sigma=\max\{1,\sigma\}$ and run
\begin{equation}
x_{t+1}=\operatorname{prox}_{\eta h}(x_t-\eta v_t),
\end{equation}
using fresh independent samples, refreshing $v_t$ every $q$ steps and otherwise using the recursive mini-batch update. If
\begin{equation}
\eta\le(8L)^{-1},\qquad
B\ge \frac{256\bar\sigma^2}{\varepsilon^2},\qquad
b\ge512q\eta^2L_{\rm ms}^2,\qquad
T\ge \frac{1024\Delta}{\eta\varepsilon^2},
\end{equation}
then, for $R$ uniform on $\{0,\ldots,T-1\}$, the computed iterate $x_{R+1}$ satisfies
\begin{equation}
\mathbb E\!\left[\operatorname{dist}(0,\nabla\Phi(x_{R+1})+\partial h(x_{R+1}))\right]\le\varepsilon .
\end{equation}
With $q=\lceil\bar\sigma/\varepsilon\rceil$ and $L_{\rm ms}=O(L)$, the number of stochastic gradient calls is
\begin{equation}
\widetilde O\!\left(\frac{\bar\sigma L\Delta}{\varepsilon^3}+\frac{\bar\sigma^2}{\varepsilon^2}\right).
\end{equation}
\end{lemma}

Apply the lemma to $F_\beta$ with $\eps_{\rm int}=c\eps/K_{\rm opt}$. Choose
\begin{equation*}
\eta\asymp L_\beta^{-1},\qquad
q\asymp \bar\sigma_\beta/\eps_{\rm int},\qquad
B\asymp \bar\sigma_\beta^2/\eps_{\rm int}^2,\qquad
b\asymp q\eta^2L_{{\rm ms},\beta}^2,\qquad
T\asymp \Delta L_\beta/\eps_{\rm int}^2 .
\end{equation*}

The refresh steps use $\lceil T/q\rceil B$ samples. The other steps use $Tb$ samples. Thus
\begin{equation*}
\left\lceil\frac{T}{q}\right\rceil B+Tb
\le
\widetilde O\!\left(
\frac{\bar\sigma_\beta L_\beta\Delta}{\eps_{\rm int}^3}
+\frac{\bar\sigma_\beta^2}{\eps_{\rm int}^2}
\right)
\le
\widetilde O\!\left(K_{\rm opt}\mathsf V\eps^{-3}\right).
\end{equation*}
The last inequality uses $\eps\in(0,1)$ and the definition of $K_{\rm opt}$. The stationarity residual is at most $c\eps/K_{\rm opt}$ in expectation. Set $\widehat\theta=\theta^+$. The residual bound and Lemma~\ref{lem:slack-residual} give $\mathbb E[\mathcal R(\widehat\theta)]\le\eps$. If each sample is weighted by its rollout length, $\mathsf V$ becomes $\mathsf W$.
\hfill\(\square\)

\subsection{Proof of Lemma~\ref{lem:self-contained-vr-subproblem}}
The proof has two steps. We first establish descent for the inexact proximal update and then control the recursive gradient error.
\par\medskip
\noindent\textbf{Step 1: Establish descent.}\par
\nopagebreak[4]
Define the exact gradient mapping $G_t$, the computed step $d_t$, and the gradient error $e_t$ by
\begin{equation*}
G_t=\frac1\eta\left(x_t-\operatorname{prox}_{\eta h}(x_t-\eta\nabla\Phi(x_t))\right),
\qquad
d_t=\frac1\eta(x_t-x_{t+1}),
\qquad
e_t=v_t-\nabla\Phi(x_t).
\end{equation*}
The proximal map is nonexpansive, so $\|d_t-G_t\|_2\le \|e_t\|_2$ and $\|G_t\|_2^2\le2\|d_t\|_2^2+2\|e_t\|_2^2$. Proximal optimality gives $d_t-v_t\in\partial h(x_{t+1})$, while smoothness gives
\begin{equation*}
\begin{aligned}
h(x_{t+1})-h(x_t)
&\le \eta\langle v_t,d_t\rangle-\frac{\eta}{2}\|d_t\|_2^2,\\
\Phi(x_{t+1})-
\Phi(x_t)
&\le -\eta\langle\nabla\Phi(x_t),d_t\rangle+\frac{L\eta^2}{2}\|d_t\|_2^2 .
\end{aligned}
\end{equation*}
Add the last two bounds and use $v_t=\nabla\Phi(x_t)+e_t$:
\begin{equation*}
F(x_{t+1})
\le
F(x_t)-\eta\left(\frac12-\frac{L\eta}{2}\right)\|d_t\|_2^2
+\eta\langle e_t,d_t\rangle .
\end{equation*}
Use $\eta\le(8L)^{-1}$ and $\langle e_t,d_t\rangle\le\frac18\|d_t\|_2^2+2\|e_t\|_2^2$. Then
\begin{equation*}
\mathbb E F(x_{t+1})
\le
\mathbb E F(x_t)-\frac{\eta}{4}\mathbb E\|d_t\|_2^2
+2\eta\mathbb E\|e_t\|_2^2 .
\end{equation*}
Sum over $t$ and use $\|G_t\|^2\le2\|d_t\|^2+2\|e_t\|^2$:
\begin{equation*}
\sum\nolimits_{t=0}^{T-1}\mathbb E\|G_t\|_2^2
\le
\frac{8\Delta}{\eta}+18\sum\nolimits_{t=0}^{T-1}\mathbb E\|e_t\|_2^2 .
\end{equation*}

\par\medskip
\noindent\textbf{Step 2: Control the recursive gradient error.}\par
\nopagebreak[4]
We next bound $e_t$. We use the following SARAH/STORM recursion \citep{Yuan2020STORMPG,Zhang2021TSIVRPG}.
\begin{lemma}[Variance-reduced gradient recursion]
\label{lem:vr-recursion}
Let $G(\theta;\xi)$ be an unbiased refresh estimator and let $D(\theta,\theta';\xi)$ be a stochastic difference estimator satisfying
\begin{equation}
\begin{aligned}
\mathbb E[G(\theta;\xi)]
&=\nabla\Phi(\theta),
&\mathbb E\|G(\theta;\xi)-\nabla\Phi(\theta)\|_2^2
&\le\sigma^2,\\
\mathbb E[D(\theta,\theta';\xi)]
&=\nabla\Phi(\theta)-\nabla\Phi(\theta'),
&\mathbb E\|D(\theta,\theta';\xi)\|_2^2
&\le L_{\rm ms}^2\|\theta-\theta'\|_2^2 .
\end{aligned}
\end{equation}
At the start $s$ of an epoch and, respectively, at a recursive step $t>s$, set
\begin{equation}
v_s=\frac1B\sum\nolimits_{j=1}^BG(\theta_s;\xi_j),
\qquad
v_t=v_{t-1}
+\frac1b\sum\nolimits_{j=1}^bD(\theta_t,\theta_{t-1};\xi_{t,j}),
\end{equation}
using fresh samples in every mini-batch. Then
\begin{equation}
\mathbb E\|v_t-\nabla\Phi(\theta_t)\|_2^2
\le
\frac{\sigma^2}{B}
+\frac{L_{\rm ms}^2}{b}
\sum\nolimits_{r=s+1}^t
\mathbb E\|\theta_r-\theta_{r-1}\|_2^2 .
\end{equation}
\end{lemma}

Apply Lemma~\ref{lem:vr-recursion} to each epoch. If an epoch starts at $s$, let $\ell_s=\min\{q,T-s\}$. Then
\begin{equation*}
\sum\nolimits_{t=s}^{s+\ell_s-1}\mathbb E\|e_t\|_2^2
\le
\frac{\ell_s\sigma^2}{B}
+\frac{qL_{\rm ms}^2}{b}
\sum\nolimits_{r=s+1}^{s+\ell_s-1}\mathbb E\|x_r-x_{r-1}\|_2^2 .
\end{equation*}
Sum over all epochs. Also use $\|x_r-x_{r-1}\|_2^2=\eta^2\|d_{r-1}\|_2^2\le2\eta^2\|G_{r-1}\|_2^2+2\eta^2\|e_{r-1}\|_2^2$. We get
\begin{equation*}
\sum\nolimits_{t=0}^{T-1}\mathbb E\|e_t\|_2^2
\le
\frac{T\sigma^2}{B}
+\frac{2q\eta^2L_{\rm ms}^2}{b}
\sum\nolimits_{t=0}^{T-1}\mathbb E\|G_t\|_2^2
+\frac{2q\eta^2L_{\rm ms}^2}{b}
\sum\nolimits_{t=0}^{T-1}\mathbb E\|e_t\|_2^2 .
\end{equation*}
Since $b\ge512q\eta^2L_{\rm ms}^2$, move the last term to the left:
\begin{equation*}
\sum\nolimits_{t=0}^{T-1}\mathbb E\|e_t\|_2^2
\le
\frac{2T\sigma^2}{B}
+\frac1{128}\sum\nolimits_{t=0}^{T-1}\mathbb E\|G_t\|_2^2 .
\end{equation*}
Insert this into the descent bound:
\begin{equation*}
\sum\nolimits_{t=0}^{T-1}\mathbb E\|G_t\|_2^2
\le
\frac{10\Delta}{\eta}+\frac{42T\sigma^2}{B}.
\end{equation*}
The choices of $B$ and $T$ give
\[
\frac1T\sum_{t=0}^{T-1}\mathbb E\|G_t\|_2^2
\le
\left(\frac{10}{1024}+\frac{42}{256}\right)\varepsilon^2,
\qquad
\frac1T\sum_{t=0}^{T-1}\mathbb E\|e_t\|_2^2
\le
\frac{2}{256}\varepsilon^2
+\frac1{128T}\sum_{t=0}^{T-1}\mathbb E\|G_t\|_2^2.
\]
Also, $\|d_t-G_t\|_2\le\|e_t\|_2$, so $T^{-1}\sum_t\mathbb E\|d_t\|_2^2\le2T^{-1}\sum_t\mathbb E(\|G_t\|_2^2+\|e_t\|_2^2)$. Proximal optimality gives
\[
\operatorname{dist}\!\left(0,\nabla\Phi(x_{t+1})+\partial h(x_{t+1})\right)
\le
(1+L\eta)\|d_t\|_2+\|e_t\|_2.
\]
Square this bound and average over uniform $R$. Use $L\eta\le1/8$ and the two bounds above. The result is at most $\varepsilon^2$. Jensen's inequality gives the stated expected residual. Finally, refresh steps use $\lceil T/q\rceil B$ samples and recursive steps use $Tb$ samples. Substitute $q=\lceil\bar\sigma/\varepsilon\rceil$ and $L_{\rm ms}=O(L)$ to get $\widetilde O(\bar\sigma L\Delta\varepsilon^{-3}+\bar\sigma^2\varepsilon^{-2})$.
\hfill\(\square\)

\subsection{Proof of Lemma~\ref{lem:vr-recursion}}
At a refresh step, independence gives $\mathbb E\|v_s-\nabla\Phi(\theta_s)\|_2^2\le\sigma^2/B$. For $t>s$, define
\begin{equation*}
\Delta_t
=
\frac1b\sum\nolimits_{j=1}^bD(\theta_t,\theta_{t-1};\xi_{t,j})
-\left(\nabla\Phi(\theta_t)-\nabla\Phi(\theta_{t-1})\right).
\end{equation*}
Thus $\mathbb E[\Delta_t\mid\mathcal F_{t-1}]=0$ and $\mathbb E[\|\Delta_t\|_2^2\mid\mathcal F_{t-1}]\le (L_{\rm ms}^2/b)\|\theta_t-\theta_{t-1}\|_2^2$. Also, $v_t-\nabla\Phi(\theta_t)=v_{t-1}-\nabla\Phi(\theta_{t-1})+\Delta_t$. The cross term has conditional mean zero. Apply this identity from $s+1$ to $t$ to get the stated bound.
\hfill\(\square\)

\subsection{Proof of Lemma~\ref{lem:slack-residual}}
\begin{proof}
Since $z_i\ge0$, we have $[g_i(\theta)]_+\le |g_i(\theta)+z_i|$, and hence $\|[g(\theta)]_+\|_1\le r_c$. The normal cone is
\begin{equation*}
N_{\mathbb R_+}(z_i)=
\begin{cases}
(-\infty,0], & z_i=0,\\
\{0\}, & z_i>0.
\end{cases}
\end{equation*}
Let $e_i=\operatorname{dist}(0,y_i+N_{\mathbb R_+}(z_i))$. If $z_i=0$, then $|[y_i]_-|\le e_i$ and $y_i z_i=0$. If $z_i>0$, then $|y_i|\le e_i$. Hence
\begin{equation*}
\|y-[y]_+\|_2\le \|e\|_2\le r_z,
\qquad
\sum\nolimits_i [y_i]_+z_i\le \sum\nolimits_i e_i z_i\le Z\sqrt m\,r_z .
\end{equation*}
Set $\lambda=[y]_+$. Then
\begin{equation*}
\left\|\nabla f(\theta)+J_g(\theta)^\top\lambda\right\|_2
\le
r_\theta+\|J_g(\theta)\|_{\mathrm{op}}\|y-\lambda\|_2
\le r_\theta+B_J r_z .
\end{equation*}

Let $e^c_i=g_i(\theta)+z_i$. Since $g_i(\theta)=e^c_i-z_i$,
\begin{equation*}
\sum\nolimits_i|\lambda_i g_i(\theta)|
\le
\sum\nolimits_i\lambda_i|e^c_i|+\sum\nolimits_i\lambda_i z_i
\le
\Lambda r_c+Z\sqrt m\,r_z .
\end{equation*}
These are the three terms in $\mathcal R(\theta)$. Adding them proves the lemma.
\end{proof}

\subsection{Proof of Lemma~\ref{lem:penalty-oracle}}
\begin{proof}
We prove the refresh bound, the recursive bound, and the smoothness bound in this order.

\par\medskip
\noindent\textbf{Step 1: Construct the refresh estimator.}\par
\nopagebreak[4]
Let $C(x,\xi)=\widehat g(\theta,\xi)+z$, whose sampled Jacobian is $J_C(x,\xi)=[\,\widehat J_g(\theta,\xi)\ \ I_m\,]$.
Use independent samples for the two factors and define
\begin{equation*}
\widehat\nabla\Phi_\beta(x;\xi_0,\xi_1,\xi_2)
=
\widehat\nabla f(\theta;\xi_0)
+\beta J_C(x,\xi_1)^\top C(x,\xi_2)
\end{equation*}
Independence gives $\mathbb E[J_C(x,\xi_1)^\top C(x,\xi_2)]=\mathbb E[J_C(x,\xi_1)]^\top\mathbb E[C(x,\xi_2)]=J_c(x)^\top c(x)$. Thus this is an unbiased estimator of $\nabla\Phi_\beta(x)=\nabla f(\theta)+\beta J_c(x)^\top c(x)$.

For the variance, use $J_C^\top C-J_c^\top c=J_C^\top(C-c)+(J_C-J_c)^\top c$. The moment bounds then give $\mathbb E\|J_C^\top C-J_c^\top c\|_2^2\le 2K_J^2K_V\mathsf V+2K_c^2K_V\mathsf V$.
Add the variance of $\widehat\nabla f$ and use $\beta\ge1$. This gives $K_\sigma^2\beta^2\mathsf V$.

\par\medskip
\noindent\textbf{Step 2: Construct the recursive difference estimator.}\par
\nopagebreak[4]
Now construct the recursive difference. Draw $\xi_1$ and $\xi_2$ independently at $x$. Let $L_{\xi_j}(x'\mid x)$ be the likelihood ratio from $x$ to $x'$. It equals one for fixed-distribution samples. Define
\begin{equation*}
\begin{aligned}
D_\beta(x,x';\xi_0,\xi_1,\xi_2)
={}&
\widehat\nabla f(x;\xi_0)
-L_{\xi_0}(x'\mid x)\widehat\nabla f(x';\xi_0)\\
&+\beta\Big[
J_C(x,\xi_1)^\top C(x,\xi_2)
-L_{\xi_1}(x'\mid x)L_{\xi_2}(x'\mid x)
J_C(x',\xi_1)^\top C(x',\xi_2)
\Big].
\end{aligned}
\end{equation*}
The two ratios change measure for the two samples. Independence gives $\mathbb E[L_{\xi_1}(x'\mid x)L_{\xi_2}(x'\mid x)J_C(x',\xi_1)^\top C(x',\xi_2)]=J_c(x')^\top c(x')$. The same identity holds for the objective-gradient term, so $\mathbb E[D_\beta(x,x';\xi_0,\xi_1,\xi_2)]=\nabla\Phi_\beta(x)-\nabla\Phi_\beta(x')$.

For the second moment, define $\mathcal P(u;\xi_1,\xi_2)=L_{\xi_1}(u\mid x)L_{\xi_2}(u\mid x)J_C(u,\xi_1)^\top C(u,\xi_2)$. The mean-value theorem gives $\|\mathcal P(x;\xi_1,\xi_2)-\mathcal P(x';\xi_1,\xi_2)\|_2\le\sup_{u\in[x,x']}\|\nabla_u\mathcal P(u;\xi_1,\xi_2)\|_{\rm op}\|x-x'\|_2$.
Each derivative acts on one likelihood ratio, $J_C$, or $C$. Assumption~\ref{ass:finite-penalty} and Lemma~\ref{lem:trajectory-difference} bound these terms. Hence
\begin{equation*}
\mathbb E\|\mathcal P(x;\xi_1,\xi_2)-\mathcal P(x';\xi_1,\xi_2)\|_2^2
\le
K^2\|x-x'\|_2^2
\end{equation*}
for a finite constant $K$. The same argument applies to the objective-gradient term. The factor $\beta$ gives the bound $K_{\rm ms}^2\beta^2\|x-x'\|_2^2$.

\par\medskip
\noindent\textbf{Step 3: Verify smoothness and the initial gap.}\par
\nopagebreak[4]
The same product expansion gives the smoothness bound from $J_C(x)^\top C(x)-J_C(x')^\top C(x')=J_C(x)^\top(C(x)-C(x'))+(J_C(x)-J_C(x'))^\top C(x')$.

Finally, item 1 of Assumption~\ref{ass:finite-penalty} gives $F_\beta(x_0)-\inf_xF_\beta(x)\le\Delta$. This proves the lemma.
\end{proof}

\subsection{Proof of Lemma~\ref{lem:stoch-tail}}
\begin{proof}
If $b_i^0<0$, the result is immediate. Now suppose $b_{i,H}\ge0$. The budget update gives $\sum_{t=0}^{H-1}w_{i,t}(s_t,a_t)\le b_i^0=\lfloor(d_i-\alpha_{\rm tail})/\eta_i\rfloor$. Also, $\gamma^t c_i(s_t,a_t)\le \eta_i w_{i,t}(s_t,a_t)$. Thus
\begin{equation*}
\sum\nolimits_{t=0}^{H-1}\gamma^tc_i(s_t,a_t)
\le
\eta_i\left\lfloor\frac{d_i-\alpha_{\rm tail}}{\eta_i}\right\rfloor
\le d_i-\alpha_{\rm tail}.
\end{equation*}
The tail after time $H$ is at most $\gamma^H/(1-\gamma)\le\alpha_{\rm tail}$. Hence $b_{i,H}\ge0$ implies $\sum_{t=0}^{\infty}\gamma^tc_i(s_t,a_t)\le d_i$. The contrapositive gives \eqref{eq:stoch-tail-inclusion}. Taking probabilities gives \eqref{eq:stoch-tail-probability}.
\end{proof}

\section{Experiment Details}
\label{sec:exp_details}

\subsection{Synthetic CCMDP diagnostics}
\label{app:synthetic-details}
\paragraph{Synthetic instance and protocol.}
The synthetic diagnostic is a finite-state absorbing discounted MDP with random absorption time. It has eight decision states \(s_0,\ldots,s_7\), one bad state \(b\), and one absorbing terminal state \(s_{\mathrm{term}}\), with initial state \(s_0\) and discount factor \(\gamma=0.95\). At each decision state there are two actions, safe \((a=0)\) and risky \((a=1)\). Both actions have self-loop probability \(0.06\), which makes the process genuinely infinite-horizon rather than fixed-horizon. The safe action has bad-transition probability \(0.002\) and moves forward with probability \(0.938\). The risky action has bad-transition probability \(p_i^r\) and moves forward with probability \(0.94-p_i^r\); the values of \(p_i^r\) and risky rewards are listed in Table~\ref{tab:synthetic_risky_params}. The bad state moves to \(s_{\mathrm{term}}\), and \(s_{\mathrm{term}}\) is absorbing.

The chance constraint is
\begin{equation*}
\mathbb P\!\left(\sum\nolimits_{t=0}^{\infty}\gamma^t c(s_t,a_t)>0.50\right)\le 0.13 .
\end{equation*}
Only the bad state has nonzero cost, \(c(b,a)=1\), while all decision and terminal states have zero cost. The safe action reward is \(0.45\) at every decision state; the bad and terminal states have zero reward. We enumerate all \(2^8=256\) stationary deterministic policies and use the same class for both learned selectors and the oracle. {
\color{black}
We use the sample sizes
\[
\begin{aligned}
\mathcal{N}_{\rm syn}={}&
\{500,1000,2000,5000,10000,20000,50000,62500,125000,250000,500000\}
\\ &{}\cup
\{10^6,2\times10^6,4\times10^6,8\times10^6\}.
\end{aligned}
\]
For each \(N\in\mathcal{N}_{\rm syn}\), we draw \(N\) next-state samples
from each of the 16 decision-state/action rows, giving a total budget
\(D=16N\).
} The bad and terminal transition rows are known. Within each repetition, both methods share one empirical kernel \(\widehat P\), reused across time, policies, safety evaluation, and reward evaluation. We use 150 independent empirical-model draws at each budget, with random seed 20260503, and evaluate selected policies exactly under the true kernel. The datasets at different budgets are drawn separately and are not nested prefixes of one sample stream.

{
\color{black}
\paragraph{Robust Bellman selector and comparator.}
The CCMDP selector applies Algorithm~\ref{alg:three-stage} with \(\rho=0.0035\),
\(\zeta=0.05\), \(d_0=3\), and
\(\alpha_{\rm tail}=\eta=0.000390625\), giving \(H=217\) and \(b^0=1279\).
Each sampled row uses the reverse-KL confidence set in~\eqref{eq:kl-confidence-set}, over all ten
successor states; the known bad and terminal rows have singleton confidence
sets. Since the bad state is followed by the absorbing terminal state,
every admissible trajectory pays the nonzero cost at most once. The rounded
violation event is exactly a bad-state visit at time \(t\leq13\), so the
robust recursion in~\eqref{eq:pessimistic-backup} reduces to a 14-layer reachability calculation.
We enumerate all 256 policies and maximize the \(H\)-step empirical Bellman
return in~\eqref{eq:reward-backup} subject to
\[
\overline q_{H,\eta}(\pi)\leq\delta-\frac{3\rho}{4}=0.127375.
\]
An empty certified set returns \textsc{Unresolved}.
\par
The Markov-CMDP comparator maximizes the same empirical return subject to
\(U_N^\pi(s_0)\leq\delta d=0.065\), where the robust expected-cost value
uses the same row-wise KL confidence sets and \(\zeta=0.05\):
\[
U_N^\pi(s)=c(s,\pi(s))+
\gamma\sup_{p\in\mathcal C_{s,\pi(s)}}p^\top U_N^\pi.
\]
Its margin above \(\mathbb E_{\widehat P}[C^\pi]\), with
\(C^\pi=\sum_{t\geq0}\gamma^t c(s_t,a_t)\), is
\(O(\sqrt{\log N/N})\) for fixed \(\gamma\). The oracle maximizes true
discounted return subject to true chance feasibility within the same
256-policy class.
\par
}

\begin{table}[h]
\centering
\small
\caption{State-dependent risky-action parameters in the synthetic CCMDP. The safe action uses reward $0.45$ and bad-transition probability $0.002$ at every decision state.}
\label{tab:synthetic_risky_params}
\renewcommand{\arraystretch}{1.12}
\setlength{\tabcolsep}{6pt}
\begin{tabular}{@{}ccccccccc@{}}
\toprule
Decision state $s_i$
& $s_0$ & $s_1$ & $s_2$ & $s_3$ & $s_4$ & $s_5$ & $s_6$ & $s_7$ \\
\midrule
Risky reward $r(s_i,1)$
& 0.72 & 0.73 & 0.74 & 0.77 & 0.79 & 0.81 & 0.84 & 0.87 \\

Risky bad probability $p_i^r$
& 0.014 & 0.018 & 0.021 & 0.024 & 0.028 & 0.032 & 0.036 & 0.040 \\

Risky next-state probability $0.94-p_i^r$
& 0.926 & 0.922 & 0.919 & 0.916 & 0.912 & 0.908 & 0.904 & 0.900 \\
\bottomrule
\end{tabular}
\end{table}

{
\color{black}
\paragraph{Synthetic results.}
Figure~\ref{fig:synthetic_post_selection} shows how the robust certificate's
conservatism decreases with additional samples. CCMDP returns no policy at
\(N=500\) or \(1000\), returns a policy in 30 of 150 trials at \(N=2000\),
and returns a policy in all 150 trials at each larger sample size.
Every returned CCMDP policy is truly feasible. At \(N=8\times10^6\), its
mean true return is \(4.39954\), compared with \(3.90793\) for Markov-CMDP
and the same-class oracle value \(4.40376\), a relative gap of approximately
\(0.096\%\). Markov-CMDP returns no policy for \(N\leq2000\), returns a
policy in 111 of 150 trials at \(N=5000\), and returns a policy in all 150
trials at each larger sample size. Every returned Markov-CMDP policy
satisfies the true expected-cost and chance constraints. Reported return means and standard
errors are conditional on a policy being returned.
\par
}

\begin{figure}[H]
    \centering
    \leavevmode\hbox{\begin{subfigure}[t]{0.33\linewidth}
            \includegraphics[width=\linewidth]{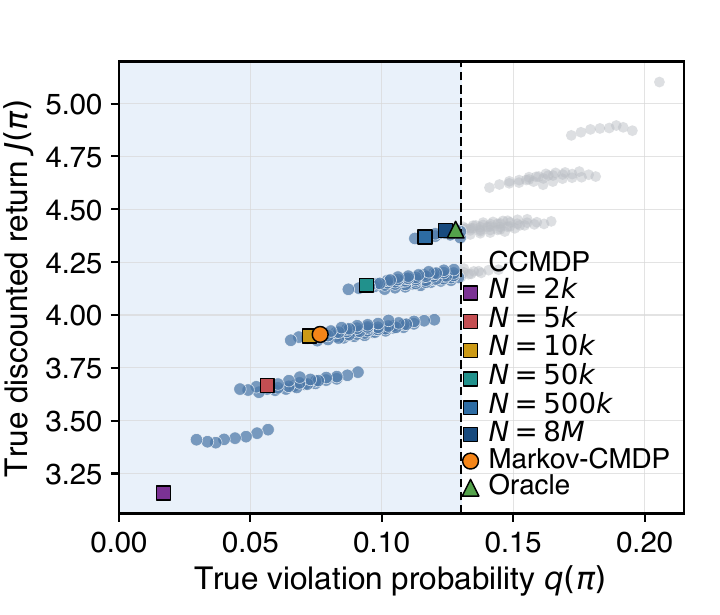}
            \caption{Policy landscape}
            \label{fig:synthetic-landscape}
        \end{subfigure}\begin{subfigure}[t]{0.33\linewidth}
            \includegraphics[width=\linewidth]{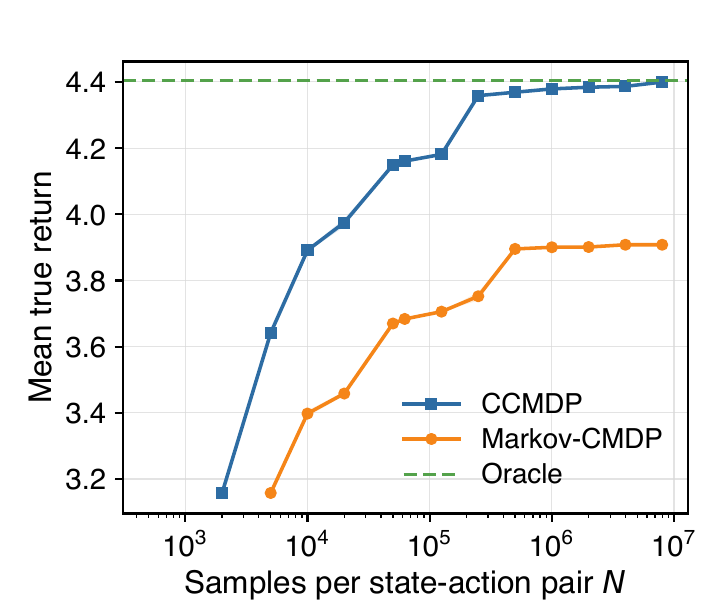}
            \caption{Return vs. samples}
            \label{fig:synthetic-return}
        \end{subfigure}\begin{subfigure}[t]{0.33\linewidth}
            \includegraphics[width=\linewidth]{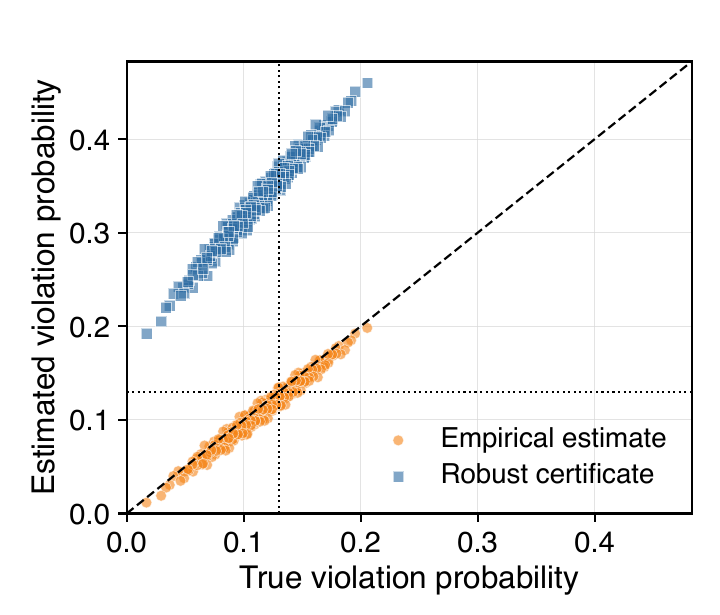}
            \caption{Certificate calibration}
            \label{fig:synthetic-calibration}
        \end{subfigure}}\caption{{\color{black}
Synthetic CCMDP diagnostics with shared transition samples.
\textbf{(a)} True return--violation landscape of all 256 stationary
deterministic policies, with CCMDP selections from repetition~0 at the
indicated sample sizes, the Markov-CMDP selection at \(N=8\times10^6\), and the
same-class oracle.
\textbf{(b)} Mean true return among returned policies from 150 independent
empirical-model draws per sample size; shaded bands show \(\pm1.96\)
empirical standard errors. Each learned transition row uses \(N\) samples,
giving total budget \(D=16N\).
\textbf{(c)} Empirical violation estimates and robust certificates versus
exact violation probabilities for all 256 policies at \(N=1000\),
repetition~0. The diagonal indicates equality, and dotted lines mark
\(\delta=0.13\).
}}
    \label{fig:synthetic_post_selection}
\end{figure}

\paragraph{Model-free experiment.}
We apply the variance-reduced policy-gradient method (VR-PG) with the recursive proximal updates of Algorithm~\ref{alg:trajectory-tail-pg} to the same synthetic instance, using a stationary Bernoulli policy with one logit per decision state. The initial risky-action probability is 0.5 at every state, and the initial nonnegative slack is zero. We optimize the normalized negative return with the quadratic slack penalty, using \(\rho=0.0035\), fixed \(\beta=80\), numerical step size 0.01, and 250,000 updates. Thus the tightened training threshold is \(\delta-2\rho=0.123\). The gradient estimator is refreshed every 20 updates using 2,048 independent trajectory triples; intervening updates use 128 triples with likelihood-ratio corrections. Each triple supplies independent reward-gradient, risk-gradient, and risk-value samples. The proximal step projects the slack onto the nonnegative half-line. Training uses 168 million trajectories. True-model evaluations are used to display the trajectory of the learned policies.

Figure~\ref{fig:synthetic_cdf_pg} shows the VR-PG optimization trajectory for training seed 11. The return rises from 4.16658 initially to 4.36245 at the last iterate, whose true violation probability is 0.12769. The best enumerated deterministic feasible return is 4.40376, and the multi-start numerical stochastic reference has return 4.50970 at violation probability 0.13.

After training, a separate random draw selects uniformly from the 250,000 post-update policies; the selected candidate is at iteration 59,611. Its true return is 4.29506 and its true violation probability is 0.12601. An independent batch of 602,267 trajectories gives a violation estimate of 0.126168; adding the validation margin \(\rho/2=0.00175\) yields 0.127918. Since this upper bound is below \(\delta=0.13\), the candidate is accepted. The selected candidate and the last iterate are marked separately in the figure.

\begin{figure}[H]
    \centering
    \leavevmode\hbox{\begin{subfigure}[t]{0.33\linewidth}
            \includegraphics[width=\linewidth]{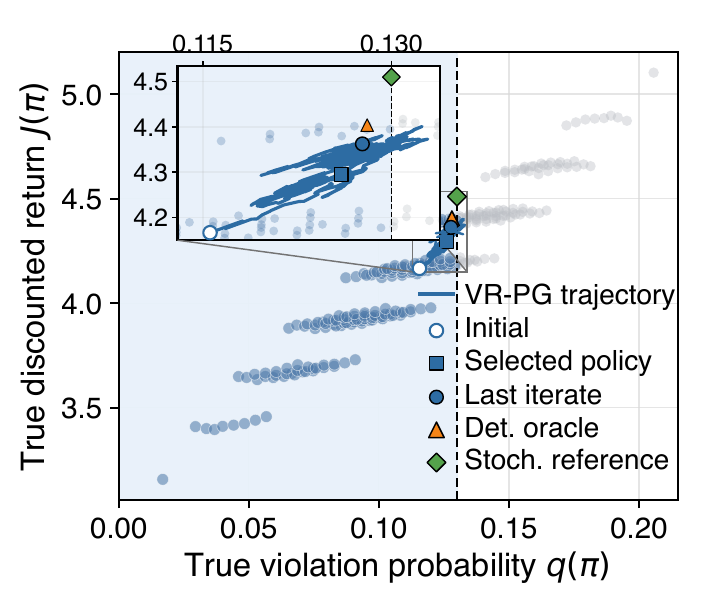}
            \caption{Policy trajectory}
            \label{fig:model-free-trajectory}
        \end{subfigure}\begin{subfigure}[t]{0.33\linewidth}
            \includegraphics[width=\linewidth]{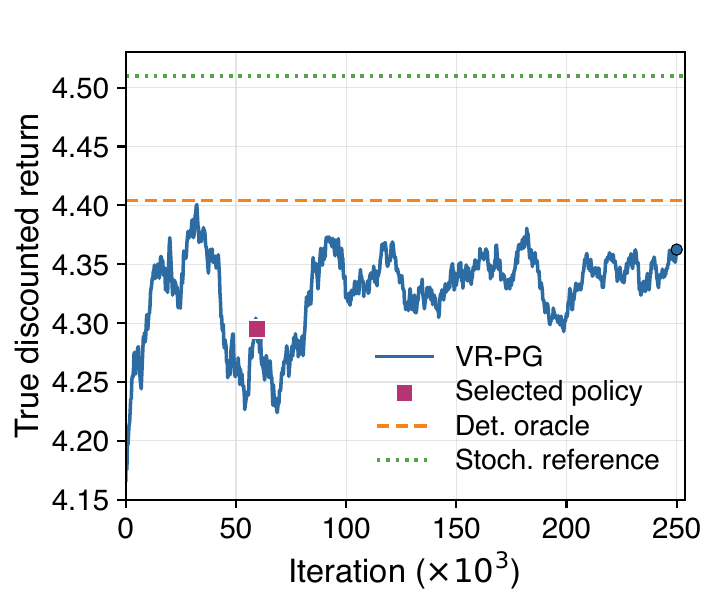}
            \caption{Return trajectory}
            \label{fig:model-free-return}
        \end{subfigure}\begin{subfigure}[t]{0.33\linewidth}
            \includegraphics[width=\linewidth]{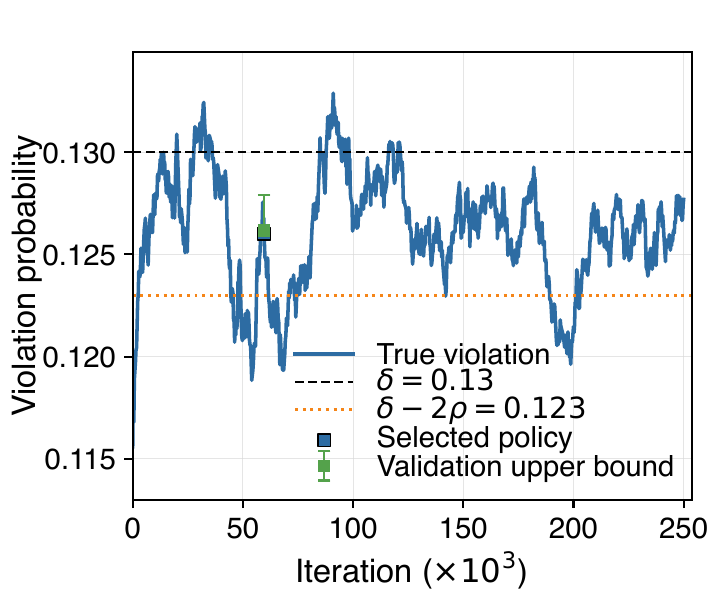}
            \caption{Risk and validation}
            \label{fig:model-free-risk}
        \end{subfigure}}\caption{
    VR-PG on the synthetic CCMDP with \(\rho=0.0035\) and \(\beta=80\).
    \textbf{(a)} Optimization trajectory in the true return--violation plane,
    with the initial policy, selected policy, and last iterate marked against
    the deterministic policy landscape. The inset enlarges the trajectory
    and includes the deterministic oracle and numerical stochastic reference.
    \textbf{(b)} True discounted return over training. The magenta square marks
    the selected policy; horizontal lines show the deterministic oracle and
    stochastic reference returns.
    \textbf{(c)} True violation probability over training, with chance threshold
    \(\delta=0.13\) and tightened training threshold
    \(\delta-2\rho=0.123\). The green marker and upward error bar show the
    independent validation estimate and its upper bound for the selected
    policy; the upper bound lies below \(\delta\).
    }
    \label{fig:synthetic_cdf_pg}
\end{figure}

\subsection{IEEE 14-bus storage-control setting}
\label{app:ieee14-storage-setting}
The storage-control experiment in Section~\ref{sec:experiments} is a finite-state CCMDP built from an IEEE 14-bus DC power-flow simulator. The state is \(s_t=(e_t,\tau_t,\ell_t)\), where \(e_t\in\{0,0.1,\ldots,1.0\}\) is storage state of charge, \(\tau_t\) is one of four time blocks, and \(\ell_t\) is one of four load regimes with multipliers \((0.82,1.00,1.18,1.42)\). Thus \(|\mathcal S|=176\). The storage device is placed at bus 14, and the action set is \(\{-12,-6,0,6,12\}\) MW, where negative actions charge and positive actions discharge. Storage capacity is \(48\) MWh, with charging and discharging efficiencies both equal to \(0.95\); infeasible commands are clipped to the SOC range and rounded to the nearest SOC grid point.

The time block advances deterministically, while the load regime follows a time-dependent Markov chain. With rows and columns ordered as low, normal, high, and extreme, the transition row is \(P_{\rm load}(\cdot\mid\tau,\ell)=0.75P_{\rm base}(\ell,\cdot)+0.25P_{\rm time}(\tau,\cdot)\), where
\begin{equation*}
P_{\rm base}=
\begin{bmatrix}
0.68&0.25&0.06&0.01\\
0.15&0.62&0.19&0.04\\
0.04&0.20&0.58&0.18\\
0.02&0.08&0.30&0.60
\end{bmatrix},
\qquad
P_{\rm time}=
\begin{bmatrix}
0.55&0.35&0.09&0.01\\
0.18&0.55&0.22&0.05\\
0.05&0.25&0.45&0.25\\
0.10&0.35&0.40&0.15
\end{bmatrix}.
\end{equation*}
The initial distribution puts SOC at \(0.5\), time at night, and load-regime probabilities \((0.15,0.70,0.13,0.02)\).

Rewards measure normalized operating benefit. The time-block prices are \((0.45,0.75,1.35,1.05)\), and the reward is an affine normalization of negative operating cost to \([0,1]\). The safety cost is normalized line-overload severity from the DC power-flow solution: if \(\rho_{\max}(s,a)\) is the maximum line-loading ratio after applying action \(a\), then \(c(s,a)=\min\{1,[\rho_{\max}(s,a)-1]_+/0.5\}\).  The chance constraint used in the reported IEEE 14-bus experiment is
\begin{equation}
\mathbb{P}_{P,\pi,\mu}\!\left(
    \sum\nolimits_{t=0}^{\infty}\gamma^t c(s_t,a_t) > 0.30
\right)
\leq 0.15,
\qquad
\gamma=0.85.
\label{eq:ieee14-chance}
\end{equation}

For the finite Bellman violation table, we use tail allowance
$\alpha_{\rm tail}=0.005$ and budget-grid width $\eta=0.0015$. By
\eqref{eq:certificate-discretization}, these choices give
\[
H=48,
\qquad
b^0=
\left\lfloor
\frac{d-\alpha_{\rm tail}}{\eta}
\right\rfloor
=196.
\]

The structured policy class contains $144$ threshold storage policies together
with the always-idle policy, for a total of $145$ policies. A threshold policy
discharges when the load regime and SOC exceed prescribed thresholds, charges
during low-price periods when the load and SOC are below prescribed
thresholds, and otherwise idles. \textcolor{black}{CCMDP and the chance oracle use
these deterministic candidates; Markov-CMDP also allows initial randomization
over them.}

{
\color{black}
For the sample-budget experiment, we use
\[
\begin{aligned}
\mathcal{N}_{\rm IEEE}={}&
\{25,50,100,200,500,1000,1500,2000,2500,3000\}
\\ &{}\cup\{3500,3750,4000,5000,10000,20000\}.
\end{aligned}
\]
\textcolor{black}{For Markov-CMDP, we extend this grid by
\(\{50{,}000,100{,}000,200{,}000,500{,}000,1{,}000{,}000\}\), with 60
independent empirical models at each added sample size.}
For each \(N\in\mathcal{N}_{\rm IEEE}\), we obtain \(N\) independent
successor samples per original state-action row, with 60 independent
empirical models per sample size. Each model uses \(D=880N\) transitions.
Its empirical kernel is shared across all time stages, rounded-budget
nodes, candidate policies, and reward evaluations.
\par
}

The main comparison is between the \textcolor{black}{robust-backward CCMDP selector} and a
\textcolor{black}{Markov-CMDP expected-cost baseline over initial mixtures of these
candidates}. We write $C_\pi = \sum_{t=0}^{\infty}\gamma^t c(s_t,a_t)$. By Markov's inequality, $\mathbb{P}_{P,\pi,\mu}(C_\pi>d) \leq \frac{\mathbb{E}_{P,\pi,\mu}[C_\pi]}{d}$, so the condition
\begin{equation}
\mathbb{E}_{P,\pi,\mu}[C_\pi]
\leq
\delta d
=
0.045
\label{eq:ieee14-markov}
\end{equation}
is sufficient for the original chance constraint. \textcolor{black}{For each
deterministic candidate \(\pi_i\), let \(u_{i,N}\) be its robust infinite-horizon
expected-cost upper bound under the same row-wise KL confidence sets
\((\zeta=0.05)\). Markov-CMDP solves the linear program
\[
\max_{w\geq0,\,\sum_iw_i=1}\ \sum_iw_i\widehat J_H(\pi_i)
\quad\text{subject to}\quad
\sum_iw_i u_{i,N}\leq0.045,
\]
where \(\widehat J_H\) is the \(H=48\)-step empirical return. An optimal
solution needs at most two components; one is drawn per trajectory and followed
throughout.
An empty certified set returns \textsc{Unresolved}.
The true-model deterministic and mixture Markov references maximize true return
subject to~\eqref{eq:ieee14-markov} over the original candidates and their initial
mixtures, respectively.} The chance-constrained oracle is the highest-return
policy in \textcolor{black}{the original deterministic}
class satisfying~\eqref{eq:ieee14-chance}.

True-model discounted returns and expected discounted safety costs are computed
from the known finite transition kernel for evaluation only. Violation
probabilities displayed in Figure~\ref{fig:ieee14-storage} are estimated under
the true transition kernel using $50{,}000$ Monte Carlo trajectories of length
$100$. These true-model quantities are not available to the learner and are
used only for evaluation and construction of the oracle references.
\textcolor{black}{Mixture statistics are the corresponding weighted averages of
component-policy statistics.}

{
\color{black}
\paragraph{Robust backward certificate.}
With \(d_0=4\) and \(\zeta=0.05\), each row uses the reverse-KL confidence
radius
\[
\kappa_N=\frac{3\log(N+1)+\log(880/0.05)}{N}.
\]
The robust backups maximize over the full 176-state simplex. We set
\(\rho=0.0035\) and solve the finite policy search exactly, maximizing the
empirical \(H\)-step return subject to
\[
\overline q_{H,\eta}(\pi)\leq\delta-\frac{3\rho}{4}=0.147375.
\]
An empty certified set returns \textsc{Unresolved}. All trials with
\(N\leq1000\) are unresolved; 59 of 60 trials return a policy at
\(N=1500\), and all 60 return a policy at each larger sample size.
The returned policies through \(N=3500\) coincide with the \textcolor{black}{deterministic} Markov-CMDP
reference. At \(N=3750\), eight of the 60 trials select the oracle,
increasing to 52 at \(N=4000\) and all 60 at \(N=5000\).
\textcolor{black}{Markov-CMDP returns no policy for \(N\leq10{,}000\) and returns
a policy in all 60 trials at each tested \(N\geq20{,}000\).}
Reported means and standard errors condition on a policy being returned.
All 599 returned \textcolor{black}{CCMDP} policies satisfy the exact rounded-risk constraint under
the true model.
\par
}

\end{document}